\documentclass[nohyperref]{article}

\usepackage[utf8]{inputenc}
\usepackage{bbm}
\usepackage{graphicx, listings, cite}

\usepackage{bbm}
\usepackage{amsmath}
\usepackage{amssymb}
\usepackage{mathtools}
\usepackage{amsthm}
\usepackage{amsfonts}

\usepackage{natbib}
\usepackage{color}
\usepackage{setspace}
\usepackage{multicol}

\usepackage{subcaption}
\usepackage{tikz, pgfplots} 
\usepackage{pgfplotstable}
\pgfplotsset{compat=1.17}
 
\usetikzlibrary{shapes,arrows,calc,trees,positioning}
\usetikzlibrary{chains,fit,shapes.geometric}

\usepackage{microtype}
\usepackage{booktabs} 

\usepackage[colorlinks,allcolors=black, citecolor=blue]{hyperref}

\usepackage[capitalize,noabbrev]{cleveref}

\theoremstyle{plain}
\newtheorem{theorem}{Theorem}[section]
\newtheorem{proposition}[theorem]{Proposition}
\newtheorem{lemma}[theorem]{Lemma}
\newtheorem{corollary}[theorem]{Corollary}
\theoremstyle{definition}
\newtheorem{definition}[theorem]{Definition}

\theoremstyle{remark}

\newcommand{\EE}[2][]{\mathbb{E}_{#1}\left[#2\right]}
\newcommand{\Var}[2][]{\mathrm{Var}_{#1}\left[#2\right]}
\newcommand{\Cov}[2][]{\mathrm{Cov}_{#1}\left[#2\right]}
\newcommand{\pr}[1]{\mathbb{P}\left[#1\right]}

\newcommand{\OSV}[1]{\sigma_V^2\left(#1\right)}

\DeclareMathOperator*{\argmin}{arg\,min}

\usepackage[textsize=tiny]{todonotes}
\usepackage[margin=1.2in]{geometry}

\title{On the Statistical Benefits of Temporal Difference Learning}

\begin{document}

\author{David Cheikhi \qquad Daniel Russo \\  Columbia University}
\date{}
\maketitle

\begin{abstract}
Given a dataset on actions and resulting long-term rewards, a direct estimation approach fits value functions that minimize prediction error on the training data. 
Temporal difference learning (TD) methods instead fit   
value functions by minimizing the degree of temporal inconsistency between estimates made at successive time-steps. Focusing on finite state Markov chains, we provide a crisp asymptotic theory of the statistical advantages of this approach. 
First, we show that an intuitive \emph{inverse trajectory pooling coefficient} completely characterizes the percent reduction in mean-squared error of value estimates. Depending on problem structure, the reduction could be enormous or nonexistent. 
Next, we prove that there can be dramatic improvements in estimates of the difference in value-to-go for two states: TD's errors are bounded in terms of a novel measure --- the problem's \emph{trajectory crossing time} --- which can be much smaller than the problem's time horizon.
\end{abstract}

\section{Introduction}
Temporal difference learning is a distinctive approach to estimation in long-term optimization problems. 
Its importance to reinforcement learning is hard to overstate. 
In their seminal book, \citet{sutton2018reinforcement} write: 
\emph{If one had to identify one idea as central and novel to reinforcement learning, it would undoubtedly be temporal difference  learning}.

Competing with temporal difference (TD) learning is a straightforward direct-estimation approach. 
There, one proceeds by collecting data on past decisions and the cumulative long-term `reward' that followed them. 
If actions were chosen with some experimental randomness, then regression of long-term rewards on the draw of actions would -- with enough data -- correctly identify actions' causal impacts. 

The direct approach has two significant drawbacks. The first is delay: actions can only be evaluated after their full long-term effects realize. The second is variance: long-term outcomes can be extremely noisy and individual actions often have a small impact on them.

TD aims to  alleviate these challenges by
leveraging data on intermediate outcomes -- those observed after the decision but before final outcomes are realized. The availability of such data is increasingly common. Robots collect regular sensor measurements, recommendation systems log sequential user interactions, and digital devices can track patient's health metrics across time. 
\citet{sutton1988learning} observed that successive predictions, updated as information is gathered, should be \emph{temporally consistent}. He proposed to fit maximally consistent value estimates by iteratively  minimizing `temporal difference errors'. TD has since become an intellectual pillar of the reinforcement learning literature. It is used in most successful applications. 

\paragraph{Our Contributions.} We aim to provide crisp insight into the statistical benefits of TD. 
This paper focuses on the simplest possible setting, where training data consists of a batch of trajectories sampled independently from a finite Markov reward process. We compare the asymptotic scaling of mean squared error under TD and direct value estimation. Two main findings, different in nature, emerge:
\begin{enumerate}
\item The relative benefits of TD are determined by a natural `inverse trajectory pooling coefficient.' TD uses value-to-go at intermediate states as a surrogate \citep{prentice1989surrogate,athey2019surrogate}. This is beneficial exactly when trajectories that originate with distinct states/actions tend to reach common intermediate states. We present simple examples illustrating when the benefits of TD are enormous and when they vanish entirely. 
\item TD is especially beneficial for advantage estimation; That is, when estimating the difference in value-to-go from one state/action versus another. TD estimates of the value-to-go at different states are not independent and the coupling of errors leads to this improvement. While the mean squared error of direct advantage estimation generally scales with the length of the problem horizon, we show that TD's errors are bounded by a smaller \emph{trajectory crossing time}. This 
 novel notion of effective horizon can be small even in some problems with unbounded  mixing time. 
\end{enumerate}

\paragraph{On the Markov assumption and state representation.}
Our focus on Markov models is standard in the academic literature on reinforcement learning. This is, in a certain technical sense, an innocuous assumption. One could always use the entire sequence of observations so far as a Markov state \citep{puterman2014markov}. 
But practical algorithms need to use appropriate compression of the history. 
The choice of representation has a subtle interplay with the benefit of TD. Indeed, we comment in section \ref{sec:open_quest} that the benefits of TD can vanish when the state representation is too rich and trajectory pooling vanishes. 
\

\section{Related works}
TD has been a central idea in RL since it was first proposed. 
It is deceptively simple, has intriguing connections to neuroscience \citep{schultz1997neural}, and seems to be routed in dynamic programming theory. 
 In the 1990s, researchers gathered both limited convergence guarantees  \citep{jaakkola1993convergence, dayan1994td}  and examples of divergence \citep{baird1995residual}. \citet{tsitsiklis1997analysis} offered a clarifying theory of when TD converges, and characterized the TD fixed point it reaches. 
\citet{maei2009convergent, sutton2009fast}  proposed methods to reach the TD fixed point in off-policy settings or when nonlinear function approximation is used. 

While illuminating, this theory does not clarify why TD should be preferred over direct value function estimation (dubbed `Monte-Carlo' or `MC').  In fact, the main guarantee is convergence to an approximate value function whose mean-squared error is \emph{larger} than the one MC reaches. 
Folklore, intuition, and experiments suggest TD often converges to its limit at a faster rate.  

The literature has emphasized the distinction between online and batch TD algorithms. The convergence speed of batch TD methods, like LSTD \citep{bradtke1996linear}, is a statistical question. With purely online algorithms, each observation is used to make a single stochastic gradient type update and then immediately discarded, so issues of memory, compute, and data efficiency are conflated. 
The deep RL literature has adopted experience replay \citep{mnih2015human, schaul2015prioritized, wang2016sample,andrychowicz2017hindsight}, which blurs the line between batch and online implementations by recording observations in a dataset and resampling them many times. 

We give a complete and intuitive characterization of the efficiency benefits of TD in the simplest possible setting: a batch variant applied without function approximation in a finite state Markov Reward Process. 
Here it is straightforward to show TD is more efficient that MC, but more subtle to understand when the efficiency gains are large. 
\citet{grunewalder2007optimality} and \citet{grunewalder2011optimal} make progress in this direction. They prove that LSTD is at least as statistically efficient as MC, without quantifying the improvement. They also display cases where the two procedures have the same performance. 
Textbooks by \citet{sutton2018reinforcement} and \citet{szepesvari2010algorithms} give illuminating examples, but no  theory. 

A number of papers bound the data requirements of TD, mainly in settings with function approximation. See for example \citep{mannor2004bias}, \citep{lu2005error}, \citep{lazaric2010finite}, \citep{pires2012statistical}, \citep{tagorti2015rate}, \citep{bhandari2018finite}, \citep{pananjady2020instance}, \citep{khamaru2020temporal}, \citep{chen2020explicit},  or \citep{farias2022markovian}. 
These show certain problem instances have low data requirements, but do not clarify when enforcing temporal consistency in value estimates produces large efficiency gains. 

Developments in Deep Reinforcement Learning, posterior to the publication of this work, have underlined the importance of trajectory pooling for statistical efficiency, under the name of trajectory stitching. For instance, \citet{ghugare2024closing} empirically shows that trajectory pooling helps generalization and suggests a data augmentation technique to obtain some of the statistical benefits of temporal consistency when using supervised learning methods (such as Monte Carlo). 

To the best of our knowledge, our insights about advantage estimation are new (See Sec~\ref{sec:horizon_truncation}).

\section{Problem Formulation}
We first describe the problem of value function estimation in Markov reward processes (MRPs). We then observe that after appropriate relabeling of the state variables, this can also represent the problem of evaluating the long-term impact of actions. 
Most mathematical results are stated for MRPs, but the alternative interpretation enriches the results.  

\subsection{Value function estimation}
\label{subsec:VFE}
A \emph{trajectory} in a terminating Markov reward process is a Markovian sequence $$\tau=(S_0, R_1, S_1, R_2, S_2, \ldots, S_{T-1}, R_T, \emptyset),$$ consisting of a sequence of states $(S_t)_{t\in [T]} \subset \mathcal{S}$, rewards $(R_t)_{t\in [T]}\subset \mathbb{R}$, and termination time $T$. The termination time is the first time at which $S_T=\emptyset$, where $\emptyset$ is thought of as a special terminal state. Assume the distribution of $R_t$ is independent of past rewards given the current state $S_t$.

A \emph{trajectory} in a terminating Markov reward process is a Markovian sequence $$\tau=(S_0, R_1, S_1, R_2, S_2, \ldots, S_{T-1}, R_T, \emptyset),$$ consisting of a sequence of states $(S_t)_{t\in [T]} \subset \mathcal{S}$, rewards $(R_t)_{t\in [T]}\subset \mathbb{R}$, and termination time $T$. The termination time is the first time at which $S_T=\emptyset$, where $\emptyset$ is thought of as a special terminal state. Assume the distribution of $R_t$ is independent of past rewards given the current state $S_t$.

The law of a Markov reward process (MRP) is specified by the tuple $\mathcal{M} = (\mathcal{S}\cup \{\emptyset\}, P, R, d)$ consisting of a state space, transition kernel, reward distribution, and initial state distribution. Here $P$ is a transition matrix over the augmented state space $\mathcal{S}\cup \emptyset$, specifying a probability $P(s' \mid s)$ of transitioning from $s$ to $s'$. We assume terminal sate is absorbing and reachable. That is, $P(\emptyset | \emptyset)=1$ and for every state $s$ there is some $t$ such that the $t$ step transition $P^{t}( \emptyset | s)$ is strictly positive. The object $R$ specifies the draw of rewards conditioned on a state transition as $R(dr | s,s') = \mathbb{P}(R_t = dr \mid S_t=s, S_{t+1}=s')$. Throughout we use the notation $r(s,s)$ for the mean of $R(\cdot | s,s')$. We assume $R(\emptyset, \emptyset)=0$. The initial state distribution $d$ is a probability distribution over $\mathcal{S}$ from which $S_0$ is drawn.

The value function 
\begin{align*}
V(s) &= \mathbb{E}\left[ \sum_{t=1}^{\infty} R_t \mid S_0=s \right] \\
&= \mathbb{E}\left[ \sum_{t=1}^{T} r(S_t, S_{t+1}) \mid S_0=s \right]
\end{align*}
specifies the expected future reward earned prior to termination. It is immediate from our formulation that $V(\emptyset)=0$. Our formulation is the Markov reward process analogue of stochastic shortest path problems \citep{bertsekas1991analysis}. Discounted problems are a special case where there is a constant probability of termination $P(\emptyset \mid s)=1-\gamma$ for each non-terminal state $s\in \mathcal{S}$. In that case, the horizon $T$ follows a geometric distribution with mean $1/(1-\gamma)$. 

A related quantity measures the value-to-go differences between states,
\[
\mathbb{A}(s,s') = V(s)-V(s').
\]
We call this the \emph{advantage} of $s$ over $s'$, since, as revealed in the next subsection, it is closely related to the advantage function in RL \citep{baird1993advantage}. 

We consider the problem of estimating the value-to-go at initial states. We compare methods that produce estimates $\hat{V}$ based on $n$ independent trajectories
\[
\mathcal{D} = \left(\tau_i = ( S_0^{(i)}, R_1^{(i)}, S_1^{(i)}, \ldots, S_{T^{(i)}-1}^{(i)}, R_{T^{(i)}}^{(i)}, \emptyset)\right)_{i=1,\dots,n}.
\]
by their mean squared error $\mathbb{E}\left[ \left(V(s) - \hat{V}(s)\right)^2 \right]$ or $\mathbb{E}\left[ \left( \mathbb{A}(s,s') - \hat{ \mathbb{A}}(s,s')\right)^2 \right]$, 
where the expectation is over the randomness in $\mathcal{D}$. We assume that all states have a non-zero probability of being visited in the dataset.

\subsection{Heterogenous treatment effect estimation}\label{subsec:HTE}

By appropriate relabeling of variables, we can interpret our problem as one of evaluating the long-term impact of a chosen decision in a specific context. 
This section demonstrates theory developed in this paper directly applies to settings seemingly more general than the one described in \ref{subsec:VFE}. Specifically, we explain how evaluating the long-term impact of a chosen decision in a specific context falls in the scope of our setting with a simple appropriate relabeling of variables.

Here we consider a special case of our formulation where the continuing state space  $\mathcal{S}=\mathcal{S}_0 \cup \mathcal{S}_I$ is partitioned into a set of initial states $\mathcal{S}_0$ and a set of intermediate states $\mathcal{S}_I$. With probability 1,  $S_0\in \mathcal{S}_0$, $S_T = \emptyset$, and at intermediate times $t\in \{1,\ldots, T-1\}$, $S_t \in \mathcal{S}_I$. 

We give the initial states a special interpretation. We think of them as consisting of an initial context $X_0$ and a decision $A_0$ and write $S_0 = (X_0, A_0)$. Using the more familiar notation $Q(X_0, A_0) \equiv V(S_0)$, we have 
\[
Q(x,a) = \mathbb{E}\left[ \sum_{t=1}^{T} R_t  \mid X_0=x, A_0=a \right].
\]

The initial distribution $d_0$ is determined by an initial context distribution $\mathbb{P}(X_0 =x)$ and a \emph{logging policy} $\mathbb{P}(A_0= x \mid X_0=x)$ which determines the frequency with which actions are observed in the data.   

Of particular interest in this setting is the advantage
\[
\mathbb{A}( (x,a), (x,a')) = Q( x,a) - Q(x, a'),
\]
which measures the performance difference between actions $a$ and $a'$ in context $x$. The term `advantage' is common in RL, but in causal inference one might call this `heterogenous treatment effect' estimation.

Since policy gradient methods typically involve computing the expectation of weighted advantages, we expect that insights developed in this paper could apply to these methods.

\section{Algorithms}

\subsection{Direct approach: First visit monte-Carlo (MC)}

For any candidate value function, we can evaluate its accuracy by comparing the future value it predicts from a given state visited in the data and the actual 
cumulative reward observed in the remainder of that trajectory. This suggests a natural \emph{direct} value estimation approach: over a candidate class of value functions, pick one that minimizes mean squared prediction error on the dataset. This method is called \emph{Monte Carlo} in the RL literature. 

To formally describe the algorithm, define the random time $T(s) = \min\{ t  : S_t = s \vee S_t = \emptyset\}$ to be the first hitting time of state $s$ if it is reached, or $T$ otherwise. Let $I(s)= \{i \in [n]: s\in \tau_{i}\}$ to be set of trajectories that visit state $s$.
Form a dataset 
$$\mathcal{D}^{\rm MC}= \bigcup_{s \in \mathcal{S}} \bigcup_{i \in I(s)} \left\{ \left(s, \sum_{t=T^{(i)}(s)+1}^{T^{(i)}} R_t^{(i)} \right) \right\} .$$
that records pairs of states and the cumulative rewards earned following the first visit to the state. 
Given a parameterized class of value functions $\{V_{\theta}: \theta \in \Theta \}$, a direct value estimation approach is to solve the least squares problem
\[ 
\min_{\theta \in \Theta } \sum_{(s, v) \in\mathcal{D}^{\rm MC}} \left( V_{\theta}(s) - v \right)^2
\]

We focus on tabular representations, where the space of parameterized value functions spans all possible functions. In that case 
\[
\hat{V}_{\rm MC}(s) = \mathbb{E}_{(S, V) \sim\mathcal{D}^{\rm MC}}\left[ V \mid S=s \right], 
\]
where the notation $(S, V) \sim\mathcal{D}^{\rm MC}$ means that state/value pairs are sampled uniformly at random from the dataset. The value estimate at state $s$ is simply the average reward earned after visits to state $s$ in the dataset. 

What we described is called \emph{first visit} Monte-carlo in the literature. It is an unbiased estimator because it only includes the first time a state was visited during the an episode. We focus on this version for analytical simplicity, but many of our key examples focus on cases where initial states are never revisited (See e.g. Subsection \ref{subsec:HTE}) and it coincides with an ``every visit'' Monte Carlo estimate. 

\subsection{Indirect approach: TD learning}
\label{sec:TD_learning}
Temporal difference learning uses a reformatted dataset consisting of tuples of reward realizations and state transitions:
\[
\mathcal{D}^{\rm TD} = \{(S^{(i)}_t, R^{(i)}_{t+1}, S^{(i)}_{t+1})_{t=1,\dots,T^{(i)}, i=1,\dots,n} \}.
\]
Define the temporal difference error between two candidate value functions $V,V'$ to be the average gap in Bellman's equation:
\[
\ell_{\rm TD}(V, V') = \mathbb{E}_{(S,R,S')\sim D^{\rm TD}}\left[ V(S) - (R+V'(S')) \right],
\]
where the notation $(S, R, S') \sim\mathcal{D}^{\rm TD}$ means that tuples are sample uniformly at random from the dataset. 
Given a parameterized class of value functions, $V_{\Theta}=\{V_{\theta}: \theta \in \Theta \}$, batch TD algorithms iteratively generate parameters $(\theta_1, \theta_2, \ldots)$ by solving $\min_{\theta_{i+1}} \ell(V_{\theta_{i+1}}, V_{\theta_{i}})$. (Online TD algorithms, combined with experience replay, make SGD updates instead.) 
For linear function approximation \citep{tsitsiklis1996analysis}, or neural networks in the `Neural tangent kernel' regime \citep{cai2019neural}, value functions are known to converge to a fixed point
\[
\hat{V}_{\rm TD} = \argmin_{V \in V_{\Theta}}  \ell\left( V, \hat{V}_{\rm TD}\right), 
\]
which, in a sense, maximizes feasible temporal consistency.

Again, we focus on tabular representations, where the space of parameterized value functions spans all possible functions. In that case, TD solves the \emph{empirical Bellman equation} 
\[ 
\hat{V}_{\rm TD}(s) = \mathbb{E}_{(S, R,S')\sim \mathcal{D}^{\rm TD}}\left[ R+\hat{V}_{\rm TD}(S') \mid S = s \right] 
\]

\section{Intuition: surrogacy and intermediate outcomes}
\begin{figure}[h!]
    \centering
    \includegraphics[width=7cm]{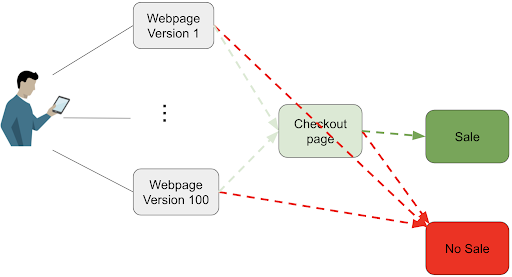}
    \caption{Modeling a user's behavior}
    \label{fig:webpage}
\end{figure}

A lot of the intuition regarding TD can be gained through the simple example in Figure \ref{fig:webpage}. 
Imagine our goal is to select the website design among 100 alternatives that leads to the largest sale rate (of some product). Users arrive and are randomly assigned to one of the 100 pages. They either click to purchase and proceed to the checkout page or navigate away from the site without purchasing. Among those who click, only a small fraction complete the sale.  

Assume, for simplicity, that we have no access to personal information that distinguishes users from one another. 
(There is only one possible $x$ in Section \ref{subsec:HTE}.)
There are 100 possible \emph{initial states}, corresponding to the webpage version, and the user is equally likely to start at each of those. A sale and a no-sale immediately precede termination. We call the checkout page an \emph{intermediate state}. The sale state is associated with a reward of 1 and all others are associated with a reward of 0. 

What we called the Monte-Carlo estimate of the value function would directly estimate the value of an impression of each webpage to be the fraction who purchased among that cohort of users. 

Due to the directed nature of state transitions, TD estimation can be thought of in two steps. We first estimate $\hat{V}(\mathtt{checkout})$ to be the fraction who purchased among users who visit the checkout page. We then estimate the value of an impression of webpage $i$ by 
\[
\hat{V}^{\rm TD}({\mathtt{webpage}\, i}) = \mathtt{CTR}(i) \times \hat{V}(\mathtt{checkout}),
\]
the empirical click-through rate among those shown webpage $i$ times the estimated sale rate on the checkout page.  

Monte-carlo estimation is unbiased, but it may be difficult to reliably estimate the efficacy of each webpage. If only a small fraction click initially, and among those who do only a small fraction convert to a sale, one would need to show each webpage to an enormous number of users. With TD, \emph{we pool data} from across users who were shown any of the 100 webpages when estimating $\hat{V}({\rm checkout})$, greatly reducing variance.
In this example, there is a lot of data pooling because trajectories that begin at distinct states quickly converge to the intermediate state.
In fact, our theory reveals that certain intuitive measures of trajectory pooling exactly determine degree of statistical benefit TD provides.  

Another view of TD is that it uses the intermediate click/no-click outcome as a surrogate or proxy-metric \citep{prentice1989surrogate}. 
Recognizing this, TD's potential downsides become as transparent as a its benefits. 
 If the conversion probability among users who visit the checkout page depends strongly on which page design they saw, the Markov property does not hold and TD is biased. 
We discuss this example again in the conclusion, mentioning how the risks and benefits interplay with the choice of state representation.

\section{Empirical illustration}\label{sec:empirical}

\subsection{The benefits of TD}
\label{subsec:TDbenef}
To illustrate how much TD can improve over MC, we explore an example: we consider a layered MRP as describe in Figure \ref{graph:layered}. A layered MRP with width $W$ and horizon $T$ has $W \times (T-1)$ states split in $T$ layers. States in layer $t$ can only transition to states in layer $t+1$ and states in the last layer $T-1$ always transition to the terminal state. 
    
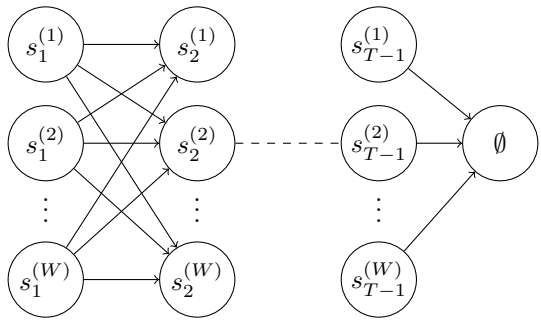
\begin{figure}[h!]
    \centering
\begin{tikzpicture}[main/.style = {draw, circle, minimum size=1cm}] 
\node[main, label=center:$s_1^{(1)}$] (11) {}; 
\node[main, label=center:$s_1^{(2)}$] (12) [below=0.3cm of 11] {}; 
\node (ellipsis1) [below=1.2cm of 11] {$\vdots$};
\node[main, label=center:$s_1^{(W)}$] (1w) [below=0.8cm of 12] {}; 

\node[main, label=center:$s^{(1)}_2$] (21) [right=1cm of 11] {}; 
\node[main, label=center:$s^{(2)}_2$] (22) [right=1cm  of 12] {}; 
\node (ellipsis2) [below=1.2cm of 21] {$\vdots$};
\node[main, label=center:$s^{(W)}_2$] (2w) [right=1cm  of 1w] {}; 



\node[main, label=center:$s^{(1)}_{T-1}$] (d1) [right=1.4cm of 21] {}; 
\node[main, label=center:$s^{(2)}_{T-1}$] (d2) [right=1.4cm of 22] {}; 
\node (ellipsisd) [below=1.2cm of d1] {$\vdots$};
\node[main, label=center:$s^{(W)}_{T-1}$] (dw) [right =1.4cm of 2w] {}; 

\node[main, label=center:$\emptyset$] (end) [right=0.6cm of d2] {};

\draw [->] (11) -- (21);
\draw [->] (12) -- (21);
\draw [->] (1w) -- (21);

\draw [->] (11) -- (22);
\draw [->] (12) -- (22);
\draw [->] (1w) -- (22);

\draw [->] (11) -- (2w);
\draw [->] (12) -- (2w);
\draw [->] (1w) -- (2w);







\draw [->] (d1) -- (end);
\draw [->] (d2) -- (end);
\draw [->] (dw) -- (end);

\draw[dashed] (22) -- (d2);

\end{tikzpicture}     

    \caption{\centering Layered MRP with width $W$ and horizon $T$. Transitions are chosen randomly and rewards are uniform on $[r(s,s')-1;r(s,s')+1]$ where $r(s,s')$ is chosen uniformly between -1 and 1.}
    \label{graph:layered}
\end{figure}

We consider a Layered MRP with width $W = 5$. We focus on state $s_1^{(1)}$ and $s_1^{(2)}$ and study the accuracy of the estimates of their value as we vary the horizon $T$ of the MRP. We also study the accuracy of the estimate of the advantage $\mathbb{A}\left(s_1^{(1)}, s_1^{(2)}\right) = V\left(s_1^{(1)}\right) - V\left(s_1^{(2)}\right)$. Figure \ref{plot:var_horizons} displays the Mean Square Error (MSE) of the TD and MC estimates for these quantities when the dataset contains $n=2000$ independent trajectories. MSE calculations involve $10000$  Monte-Carlo replications. Alongside the observed MSE, we plot projected MSE based on the central limit theorem from Proposition \ref{lem:MC_variance} and Proposition \ref{lem:TD_variance}. There is almost perfect alignment between asymptotic and finite sample results.  
\newcommand{\MyDrawVarLines}{
\addplot [blue,mark=*,only marks, mark options={scale=.5}, error bars/.cd, y explicit,y dir=both] table[x=num_samples,y=Empirical_treatment_TD_A, y error plus expr=\thisrow{Empirical_treatment_TD_A_UB} - \thisrow{Empirical_treatment_TD_A} , y error minus expr=\thisrow{Empirical_treatment_TD_A} - \thisrow{Empirical_treatment_TD_A_LB}] {Data/var_horizons.csv};
\addplot [blue] table[x=num_samples,y=Theoretical_treatment_TD_A] {Data/var_horizons.csv};

\addplot [brown,mark=*,only marks, mark options={scale=.5}, error bars/.cd, y explicit,y dir=both] table[x=num_samples,y=Empirical_treatment_TD_B, y error plus expr=\thisrow{Empirical_treatment_TD_B_UB} - \thisrow{Empirical_treatment_TD_B} , y error minus expr=\thisrow{Empirical_treatment_TD_B} - \thisrow{Empirical_treatment_TD_B_LB}] {Data/var_horizons.csv};
\addplot [brown] table[x=num_samples,y=Theoretical_treatment_TD_B] {Data/var_horizons.csv};

\addplot [red,mark=*,only marks, mark options={scale=.5}, error bars/.cd, y explicit,y dir=both] table[x=num_samples,y=Empirical_treatment_TD_ATE, y error plus expr=\thisrow{Empirical_treatment_TD_ATE_UB} - \thisrow{Empirical_treatment_TD_ATE} , y error minus expr=\thisrow{Empirical_treatment_TD_ATE} - \thisrow{Empirical_treatment_TD_ATE_LB}] {Data/var_horizons.csv};
\addplot [red] table[x=num_samples,y=Theoretical_treatment_TD_ATE] {Data/var_horizons.csv};

\addplot [blue,mark=x,only marks, mark options={scale=1}, error bars/.cd, y explicit,y dir=both] table[x=num_samples,y=Empirical_treatment_MC_A, y error plus expr=\thisrow{Empirical_treatment_MC_A_UB} - \thisrow{Empirical_treatment_MC_A} , y error minus expr=\thisrow{Empirical_treatment_MC_A} - \thisrow{Empirical_treatment_MC_A_LB}] {Data/var_horizons.csv};
\addplot [blue, dashed] table[x=num_samples,y=Theoretical_treatment_MC_A] {Data/var_horizons.csv};

\addplot [brown,mark=x,only marks, mark options={scale=1}, error bars/.cd, y explicit,y dir=both] table[x=num_samples,y=Empirical_treatment_MC_B, y error plus expr=\thisrow{Empirical_treatment_MC_B_UB} - \thisrow{Empirical_treatment_MC_B} , y error minus expr=\thisrow{Empirical_treatment_MC_B} - \thisrow{Empirical_treatment_MC_B_LB}] {Data/var_horizons.csv};
\addplot [brown, dashed] table[x=num_samples,y=Theoretical_treatment_MC_B] {Data/var_horizons.csv};

\addplot [red,mark=x,only marks, mark options={scale=1}, error bars/.cd, y explicit,y dir=both] table[x=num_samples,y=Empirical_treatment_MC_ATE, y error plus expr=\thisrow{Empirical_treatment_MC_ATE_UB} - \thisrow{Empirical_treatment_MC_ATE} , y error minus expr=\thisrow{Empirical_treatment_MC_ATE} - \thisrow{Empirical_treatment_MC_ATE_LB}] {Data/var_horizons.csv};
\addplot [red, dashed] table[x=num_samples,y=Theoretical_treatment_MC_ATE] {Data/var_horizons.csv};
}
\begin{figure}[h!]
\begin{subfigure}[h]{\columnwidth}
\centering    
\begin{tikzpicture}
\begin{axis}[
            xmin=0,xmax=130,
            ymin=0.0,ymax=0.5,
            width=7.5cm,
            height=7.5cm,
            table/col sep=comma,
            xlabel = Horizon $T$,
            ylabel = MSE,
            grid=both] 
    \MyDrawVarLines
\end{axis}
\end{tikzpicture}
\caption{Full Y-Axis}
\end{subfigure}
\begin{subfigure}[h]{\columnwidth}
\centering
\begin{tikzpicture}
\begin{axis}[
            xmin=0,xmax=130,
            ymin=0.0,ymax=0.05,
            width=7.5cm,
            height=7.5cm,
            table/col sep=comma,
            xlabel = Horizon $T$,
            ylabel = MSE,
            grid=both,
            scaled y ticks=false,
            legend pos=north east,
            title style={at={(0.5,0)},anchor=north,yshift=-0.1}] 
            \MyDrawVarLines
\end{axis}
\end{tikzpicture}
\caption{Truncated Y-Axis}
\end{subfigure}
\center
\begin{tikzpicture}
\begin{axis}[
            xmin=10,xmax=130,
            ymin=0.0,ymax=0.5,
            width=2cm,
            height=2cm,
            hide axis,
            table/col sep=comma,
            xlabel = Horizon $H$,
            ylabel = Variance,
            ylabel near ticks, yticklabel pos=right,
            grid=both,
            legend style={anchor=north west,
                legend cell align=left,
                /tikz/every even column/.append style={column sep=0.5cm}},
            legend columns = 2] 
\addlegendimage{fill = blue, area legend}
\addlegendentry{Value at $s$}
\addlegendimage{black}
\addlegendentry{Asymptotic TD MSE}
\addlegendimage{fill = brown, area legend}
\addlegendentry{Value at $s'$}
\addlegendimage{dashed,black}
\addlegendentry{Asymptotic MC MSE}

\addlegendimage{fill = red, area legend}
\addlegendentry{Advantage}
\addlegendimage{mark=*, only marks,black}
\addlegendentry{Empirical TD MSE}
\addlegendimage{opacity = 0}
\addlegendentry{}
\addlegendimage{mark=x, only marks,black}
\addlegendentry{Empirical MC MSE}

\end{axis}
\end{tikzpicture}

    \caption{\centering MSE of different MC and TD estimates on Layered MRP with $W = 5$ and varying horizon $T$}
    \label{plot:var_horizons}
\end{figure}

We highlight three main take-aways from this example:
\begin{enumerate}
    \item {\bf TD can vastly outperform MC}. For the chosen states $s_1^{(1)}$ and $s_1^{(2)}$, the MSE is about 5 times smaller when using TD instead of MC for a Layered MRP with width $W = 5$ and horizon $T = 120$. 
    \item {\bf TD benefits are enhanced for advantage estimation}. In this example, TD performs up to 100 times better than MC for the advantage estimation. This example also shows that the MSE of the TD estimate of the ATE is smaller than the MSE of individual estimates of the value of $s_1^{(1)}$ and $s_1^{(2)}$ when the horizon $T$ is larger than 20. On the other hand, the MSE of the MC estimate of the advantage is larger than the individual MSE of the estimates of the value.
    \item {\bf TD effectively truncates the horizon}. While the MSE of the MC estimate of the advantage scales linearly with the horizon $T$, the MSE of the TD estimate is constant with respect to the horizon $T$. This is all the more striking that the variance of the total reward along a trajectory scales linearly with the horizon. 
\end{enumerate}

\subsection{Dependence on the MRP structure}

We have seen in Section \ref{subsec:TDbenef} that TD vastly outperforms MC in the case of Layered MRP. However, different MRP structures lead to different level of improvement of TD over MC. 
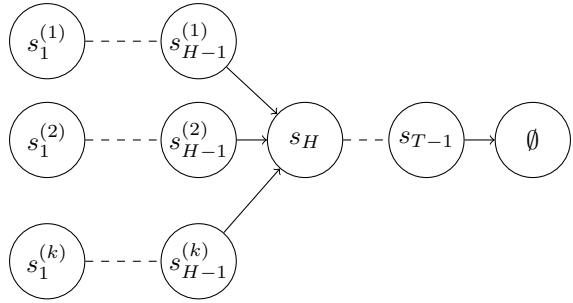
\begin{figure}[h!]
\centering
\begin{tikzpicture}[main/.style = {draw, circle, minimum size = 1cm}] 
\node[main, label=center:$s_1^{(1)}$] (11) {}; 
\node[main, label=center:$s_1^{(2)}$] (12) [below=0.3cm of 11] {}; 
\node[main, label=center:$s_1^{(k)}$] (1k) [below=0.6cm of 12] {}; 

\node[main, label=center:$s^{(1)}_{H-1}$] (H1) [right=1cm of 11]{}; 
\node[main, label=center:$s^{(2)}_{H-1}$] (H2) [below=0.3cm of H1] {}; 
\node[main, label=center:$s^{(k)}_{H-1}$] (Hk) [below=0.6cm of H2] {}; 

\node[main, label=center:$s_{H}$] (H+1) [right=0.4cm of H2] {}; 
\node[main, label=center:$s_{T-1}$] (T) [right=0.6cm of H+1]{}; 

\node[main, label=center:$\emptyset$] (end) [right=0.4cm of T]{};

\draw [dashed] (11) -- (H1);
\draw [dashed] (12) -- (H2);
\draw [dashed] (1k) -- (Hk);

\draw [->] (H1) -- (H+1);
\draw [->] (H2) -- (H+1);
\draw [->] (Hk) -- (H+1);

\draw [dashed] (H+1) -- (T);
\draw [->] (T) -- (end);

\end{tikzpicture} 
\caption{MRP with meeting horizon $H$}
\label{graph:meeting_hor_teaser}
\end{figure}
To illustrate this, we introduce a new class of MRPs described in Figure \ref{graph:meeting_hor_teaser}. Each of the $k$ initial states $s_1^{(1)}, \dots, s_1^{(k)}$ lead to disjoint trajectories for the first $H-1$ steps before reaching a common state on the $H$th step. We are interested in seeing how TD and MC perform when the crossing time $H$ varies. Figure \ref{plot:meeting_horizons} displays the ratio of the MSE of the TD and MC estimates for the values of $s_1^{(1)}, s_1^{(2)}$ and for the advantage as the crossing time $H$ varies. The MSE used to compute the ratios have been computed using $n=200$ independent trajectories and $1000$ Monte-Carlo replications. These ratios are plotted alongside the asymptotic ratio from Theorem \ref{thm:state_coeff}.  As the crossing time $H$ gets closer to the horizon $T$, the advantage of TD over MC vanishes until $H = T$, when TD and MC produce the exact same estimates.
\begin{figure}[h!]
\begin{center}
\begin{tikzpicture}
\begin{axis}[
            title={},
            xmin=0,xmax=21,
            ymin=0,ymax=1,
            width=8cm,
            height=8cm,
            table/col sep=comma,
            xlabel = Meeting horizon $H$,
            ylabel = Ratio of MSE $\EE{(\hat{V}_{TD}-V)^2}/\EE{(\hat{V}_{MC}-V)^2}$,
            grid=both,
            legend pos=south east]   

\addlegendimage{fill = blue, area legend}
\addlegendentry{Value at $s$}
\addlegendimage{fill = brown, area legend}
\addlegendentry{value at $s'$}
\addlegendimage{fill = red, area legend}
\addlegendentry{Advantage}
\addlegendimage{black}
\addlegendentry{Asymptotic ratio}
\addlegendimage{mark=*, only marks,black}
\addlegendentry{Empirical ratio}

\addplot [blue,mark=*,only marks, mark options={scale=0.5}, error bars/.cd, y explicit,y dir=both] table[x=meeting_hor,y=state_s_ratio, y error plus expr=\thisrow{state_s_ratio_UB} - \thisrow{state_s_ratio} , y error minus expr=\thisrow{state_s_ratio} - \thisrow{state_s_ratio_LB}] {Data/meeting_hor.csv};
\addplot [blue] table[x=meeting_hor,y=state_s_ratio_theo] {Data/meeting_hor.csv};

\addplot [brown,mark=*,only marks, mark options={scale=0.5}, error bars/.cd, y explicit,y dir=both] table[x=meeting_hor,y=state_sb_ratio, y error plus expr=\thisrow{state_sb_ratio_UB} - \thisrow{state_sb_ratio} , y error minus expr=\thisrow{state_sb_ratio} - \thisrow{state_sb_ratio_LB}] {Data/meeting_hor.csv};
\addplot [brown] table[x=meeting_hor,y=state_sb_ratio_theo] {Data/meeting_hor.csv};

\addplot [red,mark=*,only marks, mark options={scale=0.5}, error bars/.cd, y explicit,y dir=both] table[x=meeting_hor,y=state_ATE_ratio, y error plus expr=\thisrow{state_ATE_ratio_UB} - \thisrow{state_ATE_ratio} , y error minus expr=\thisrow{state_ATE_ratio} - \thisrow{state_ATE_ratio_LB}] {Data/meeting_hor.csv};

\addplot [red] table[x=meeting_hor,y=state_ATE_ratio_theo] {Data/meeting_hor.csv};

\end{axis}
\end{tikzpicture}
\end{center}
\caption{Ratio of variance between TD and MC as a function of the meeting horizon $H$ for $T = 20$}
\label{plot:meeting_horizons}
\end{figure}
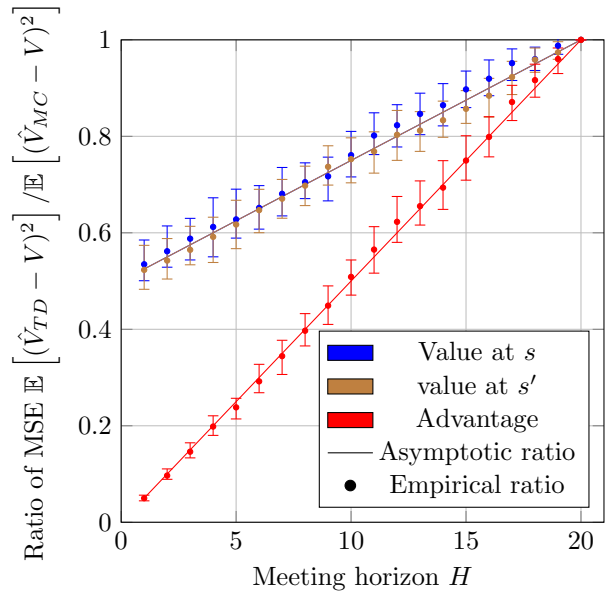
To convey intuition, we first focus on the two extreme cases:
\begin{itemize}
    \item In the case where $H=2$, depicted in Figure \ref{fig:good_instance}, all initial states directly transition to the same state. This mimics the webpage example in Figure \ref{fig:webpage}. In this example, apart from the first reward, trajectories do not depend of the initial state. TD pools trajectories across actions which allows to highly reduce the variance of the estimate of the value at $s_2$. This low variance estimate is then used as a surrogate to estimate the value-to-go at initial states. On the other hand, MC does not leverage the structure of the MRP and produces an independent estimate for each initial state, using only trajectories starting at a given state to evaluate this state. In this case, TD will significantly improve over MC.
    \item  At the other extreme, consider $H=T$. Then, no two initial states can lead to a common state before the terminal state, as shown in Figure \ref{fig:bad_instance}. There is no opportunity to pool trajectories across actions so TD strictly reduces to MC in this case.
\end{itemize}
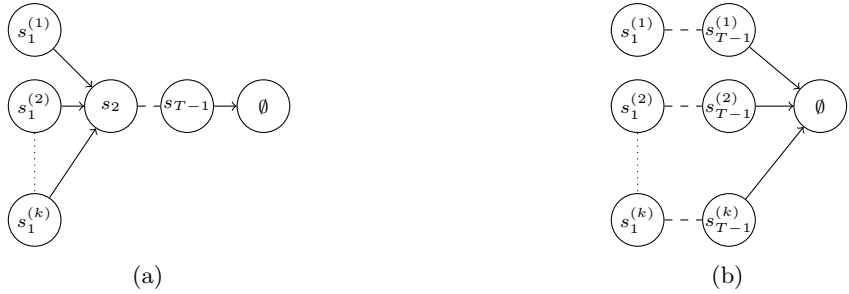
\begin{figure}[h!]
\centering
\begin{subfigure}[h]{0.49\columnwidth}
\centering
\begin{tikzpicture}
[main/.style = {draw, circle, minimum size = 0.7cm},
every label/.append style={font=\scriptsize},
] 

\node[main, label=center:$s^{(1)}_1$] (1) {}; 
\node[main, label=center:$s^{(2)}_1$] (2) [below=0.3cm of 1] {}; 
\node[main, label=center:$s^{(k)}_1$] (4) [below=0.8cm of 2] {}; 
\node[main, label=center:$s_2$] (5) [right=0.3cm of 2] {}; 

\node[main, label=center:$s_{T-1}$] (7) [right=0.3cm of 5] {}; 
\node[main, label=center:$\emptyset$] (end) [right=0.3cm of 7]{};

\draw [dotted] (2) -- (4);
\draw [->] (1) -- (5);
\draw [->] (2) -- (5);
\draw [->] (4) -- (5);

\draw [dashed] (5) -- (7);
\draw [->] (7) -- (end);

\end{tikzpicture} 
\caption{}
\label{fig:good_instance}
\end{subfigure}
\begin{subfigure}[h]{0.49\columnwidth}
\centering
\begin{tikzpicture}
[main/.style = {draw, circle, minimum size = 0.7cm},
every label/.append style={font=\scriptsize},
] 
\node[main, label=center:$s^{(1)}_1$] (11) {}; 
\node[main, label=center:$s^{(2)}_1$] (12) [below=0.3cm of 11] {}; 
\node[main, label=center:$s^{(k)}_1$] (1k) [below=0.8cm of 12] {}; 

\node[main, label=center:$s^{(1)}_{T-1}$] (m1) [right=0.5cm of 11] {}; 

\node[main, label=center:$s^{(2)}_{T-1}$] (m2) [right=0.5cm of 12] {}; 

\node[main, label=center:$s^{(k)}_{T-1}$] (mk) [right=0.5cm of 1k] {}; 
\node[main, label=center:$\emptyset$] (end) [right=0.5cm of m2]{};

\draw [dotted] (12) -- (1k);
\draw [dashed] (11) -- (m1);

\draw [dashed] (12) -- (m2);

\draw [dashed] (1k) -- (mk);

\draw [->] (m1) -- (end);
\draw [->] (m2) -- (end);
\draw [->] (mk) -- (end);

\end{tikzpicture} 
\caption{}
\label{fig:bad_instance}
\end{subfigure}
    \caption{ \ref{fig:good_instance} is an instance on which TD leverages pooling to improve considerably over MC. \ref{fig:bad_instance} is an instance on which TD and MC output the same estimate.}
    \label{fig:graph_examples}
\end{figure}

\subsection{Organization of the results}

In Section \ref{sec:pooling_coefficient}, we characterize the ratio in variance between TD and MC estimates for value estimation depending on the MRP structure. This characterization enables an intuitive understanding of which MRP structures lead to a large improvement of TD and, conversely, for which MRP structures TD and MC perform similarly. In Section \ref{sec:horizon_truncation}, we show that the TD estimate of advantages scales with an effective horizon that can be much smaller than the horizon.

\section{Value estimation and the pooling coefficient}
\label{sec:pooling_coefficient}

Recall that $T(s)$ is the first hitting time of state $s$ if it is reached, or $T$ otherwise. The variables   
\[
N(s')=\sum_{t=0}^{T} \mathbbm{1}(S_t=s'),    \; \;   N(s \to s') = \sum_{t=T(s)}^{T} \mathbbm{1}(S_t=s'), 
\]
respectively measure the total number of visits to state $s'$ and the number of visits to $s'$ which occur after a visit to state $s$. 

Define the coupling coefficient between $s$ and $s'$ by
\[
C(s, s') = \frac{\mathbb{E}\left[ N(s \to s')\right] }{\mathbb{E}\left[ N(s') \right]},
\]
with $C(s,s')=0$ if $\EE{ N(s')}=0$.
Implicitly, it is understood that $S_0$ is drawn from the MRP's initial distribution $d$.
Among all trajectories which reach state $s'$, the coupling coefficient measures the fraction which first pass through state $s$.  
If the coupling coefficient is large, it means $s$ and $s'$ are strongly coupled. 

The inverse \emph{trajectory pooling coefficient} measures the average  coupling coefficient $C(s, s')$ over a distribution of possible successor states $s'$. The right distribution over which to average turns out to be $\mu_{s}(\cdot)$, identified in the definition below.  That distribution weights highly states that are likely to be visited following a visit to $s$ (high $\mathbb{E}[N(s') \mid S_0=s]$) and contribute heavily to estimator variance (measured through the one-step variance ${\rm Var}\left(R_{t+1} + V(S_{t+1}) \mid S_t = s'\right)$).

\begin{definition}[Inverse trajectory pooling coefficient]
    \label{def:traj_pooling_coeff}
   For any state $s\in \mathcal{S}$ define 
    \begin{equation*}
                C(s) = \EE[s' \sim \mu_{s}]{ C(s,s')},
    \end{equation*}
    where $\mu_{s}(\cdot)$ a probability distribution over states defined by
    \[
    \mu_{s}(s') \propto \mathbb{E}\left[ N(s') \mid S_0=s \right] \times \Var{R_t + V(S_{t+1}) \mid S_t=s'}.  
    \]
\end{definition}
Again, the inverse trajectory pooling coefficient is small when there is a lot of trajectory pooling. 
The next theorem compares the asymptotic mean squared error of the value estimated under TD and a direct approach. 
The asymptotic ratio of their mean squared errors is equal to the inverse trajectory pooling coefficient. 
\begin{theorem}
    \label{thm:state_coeff}
    For any $s \in \mathcal{S}$,
    \begin{equation*}
         \lim_{n\rightarrow\infty} \dfrac{\EE{\left(\hat{V}_{TD}(s) - V(s)\right)^2}}{\EE{\left(\hat{V}_{MC}(s) - V(s)\right)^2}} = C(s).
    \end{equation*}
\end{theorem}
Let us interpret this result. Recall that TD updates value prediction at state $s$ using value predictions at successor states. The theorem shows this is helpful precisely when there is a lot of trajectory pooling, resulting in a small inverse trajectory pooling coefficient.
When this holds, and the dataset $\mathcal{D}$ is large, there will be many trajectories which reach an important possible successor $s'$ of $s$, but never cross $s$ first. TD leverages these trajectories to learn about $s'$ and then properly incorporates that knowledge to better evaluate $s$. Direct estimation approaches only use sub-trajectories originating at $s$ to evaluate $s$ and forego the trajectory pooling benefit.  
We already developed this intution by discussing the simple example in Figure \ref{fig:webpage}. 
The theorem confirms that this interpretation of TD's advantages is exactly the right one.  

Figures \ref{fig:good_instance} describes an instance with extreme trajectory pooling. Trajectories that start in distinct states tend to immediately reach common successors, so TD understands value-to-go from successors quite well. Figure \ref{fig:bad_instance} is a case with no trajectory pooling at all (i.e. $C(s)=1$).

\section{Horizon truncation in advantage estimation}
\label{sec:horizon_truncation}
Section \ref{sec:empirical} previewed two of the paper's key insights: TD's benefits are enhanced for advantage estimation and, in that setting, it effectively truncates the problem's time horizon.
Theory in this section formalizes these insights. 

\paragraph{The MSE of direct advantage estimates scales with the horizon.} The variance of the total reward along a trajectory typically scales with the horizon. Therefore, it would not be surprising that the mean squared error of the estimate of the advantage also scales with the horizon. We show that it is indeed the case for MC by stating a lower bound on the mean squared error.

\begin{proposition}
    \label{prop:MC_lower_bound}
    For $s,s' \in \mathcal{S}$ such that $\pr{s \in \tau \wedge s' \in \tau} = 0$,
    \begin{align*}
        &\lim_{n \to \infty} n \cdot \EE{\left(\hat{\mathbb{A}}_{\rm MC}(s,s') - \mathbb{A}(s,s')\right)^2} \geq \sigma_{\rm \min}^2\left(\frac{\EE{T | S_0 = s}}{\pr{s\in \tau}} + \frac{\EE{T | S_0 = s'}}{\pr{s' \in \tau}} \right).
    \end{align*}
where $\sigma_{\min}^2 = \min_{s \in \mathcal{S}} \Var{R_t + V(S_{t+1}) \mid S_t=s}$
\end{proposition}

The condition $\pr{s \in \tau \wedge s' \in \tau} = 0$ guarantees that no trajectory can visit both $s$ and $s'$. It ensures that the MC estimate of the value at $s$ and $s'$ are independent. It is verified when considering the heterogeneous treatment effect, described in Section \ref{subsec:HTE}, where a single action is chosen for every trajectory. The scaling in the inverse probability of visiting $s$ and $s'$ appears because $n\pr{s \in \tau}$ and $n\pr{s' \in \tau}$ are asymptotically the number of trajectories available for the Monte-Carlo estimation.  

\paragraph{The MSE of TD's advantage estimates scales with a smaller trajectory crossing time.}
Rather than scale with problem's time horizon, the mean squared error of TD's advantage estimate is bounded by a smaller notion of the problem's effective horizon. To formally capture this phenomenon, we introduce the  \emph{trajectory crossing time} $H(s,s')$ for two states $s$ and $s'$ to be the expected time for trajectories starting at $s$ and $s'$ to cross under the most optimistic coupling. Two trajectories always cross once both have terminated, as in that case both have visited the terminal state $\emptyset$, but they could cross much sooner. Intuition for the this definition is provided below. 

\begin{definition}
   The set of distributions $\Psi(s,s')$ is the set of all joint distributions over trajectories $(\tau, \tilde{\tau})$ such that the marginal distributions of $\tau$ and $\tilde{\tau}$ are those of trajectories starting at $s$ and $s'$, respectively.
\end{definition}

\begin{definition}[Trajectory crossing time]
   The trajectory crossing time of two states $s$ and $s'$ is the expected time for trajectories originating from $s$ and $s'$ to cross under the best coupling that preserves the trajectories' marginal distributions:
   \begin{equation*}
       H(s,s') = \min_{\psi \in \Psi(s,s')} \EE[(\tau, \tilde{\tau}) \sim \psi]{\inf \{t | \mathcal{C}_t(S, S') \neq \emptyset \}}
   \end{equation*}
   where $\mathcal{C}_t(S, S') = \{S_0, \dots, S_t\} \cap \{S_1', \dots, S_t'\}$ is the set of states visited by both trajectories at time $t$.
\end{definition}
The following theorem establishes that the mean squared error of the TD estimate of the advantage scales with the trajectory crossing time instead of the full horizon:
\begin{theorem}
\label{thm:horizon_truncation}
For $s,s' \in \mathcal{S}$,
\begin{align*}
    &\lim_{n\rightarrow\infty} n\cdot \EE{\left(\hat{\mathbb{A}}_{\rm TD}(s,s') - \mathbb{A}(s,s')\right)^2} \leq  2\left(\dfrac{ \sigma^2_{\rm max}}{\min\left(\pr{ s\in \tau}, \pr{s' \in \tau} \right)}\right) \cdot H(s,s')
\end{align*}
with $ \sigma^2_{\rm max} = \max_{s \in \mathcal{S}} \Var{R_t + V(S_{t+1}) \mid S_t=s}$.
\end{theorem}
Figure \ref{plot:var_horizons}, in Section \ref{sec:empirical}, provides an empirical illustration of this result. 

This result can actually understate the benefits the TD. 
Any trajectory pooling that happens before the trajectories cross further helps reducing the variance, 
but is not reflected in the upper bound of Theorem \ref{thm:horizon_truncation}. In particular, we show in Appendix \ref{subsec:TD_improves} that trajectory pooling ensures that TD estimates are at least as accurate as MC estimates, even when the crossing time matches the full horizon. We also discuss the tightness of these bounds in Appendix \ref{subsec:tight_bounds}.

\paragraph{Comparison with a coupling time.}
 Trajectories are said to cross if one of the trajectories reaches a state already visited by the other one, potentially at an earlier time. It is different from the coupling time where trajectories have to reach a common state simultaneously. In particular, the crossing time is always smaller than the coupling time.The MRP defined in Figure \ref{fig:meeting_horizon} illustrates this. Trajectories starting at states $s_1^{(1)}$ and $s_1^{(2)}$ only couple when reaching the terminal state, after $m+1$ timesteps. However, they cross in two timesteps, that is $H\left(s_1^{(1)},s_1^{(2)}\right) = 2$.     
\begin{figure}[t!]
\centering
\begin{tikzpicture}[main/.style = {draw, circle, minimum size = 1cm}] 
\node[main] (1) {$s_1^{(1)}$}; 
\node[main] (2) [below=0.4cm of 1] {$s_1^{(2)}$}; 
\node[main] (5) [right=0.5cm of 1] {$s_2$}; 
\node[main] (6) [right=0.5cm of 5] {$s_3$}; 
\node[main] (7) [right=1cm of 6] {$s_m$}; 

\draw [->] (1) -- (5);
\draw [->] (2) to [bend right = 30] (6);
\draw [->] (5) -- (6);

\draw [dashed] (6) -- (7);
\end{tikzpicture} 
\caption{An example with no coupling but rapid crossing.}
\label{fig:meeting_horizon}
\end{figure}
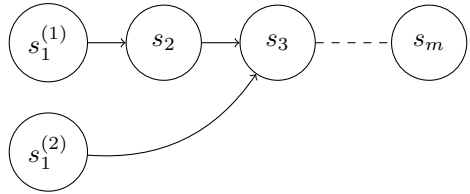

\paragraph{Intuition for the result.}
 
 Let us give intuition for Theorem \ref{thm:horizon_truncation}. 
 Under the MRP structure, two trajectories reaching a common state have the same future expected reward. Hence, when estimating the difference in expected total rewards along two trajectories, one starting at $s$, the other starting at $s'$, it is only useful to estimate them up until the state where trajectories cross. By computing estimates at every state, TD leverages this property: two trajectories reaching a common state (the crossing state) use the same estimate (the value at the crossing state) to update predictions along both trajectories. Since the same estimate is used, its value cancels out when computing the difference in values at $s$ and $s'$. Whether the value at the crossing state is accurately estimated doesn't affect the estimation of the advantage.

\section{Open questions}
\label{sec:open_quest}

 There is a subtle interplay between the choice of state representation and the benefits of imposing temporal consistency in value estimates. 
 Consider again the problem in Figure \ref{fig:webpage}. 
 In that case, we chose to represent the `checkout page' as a state, implying that the purchase probability at the checkout page does not depend on the initial webpage shown to the user. 
 This makes a strong surrogacy assumption, which TD leverages to greatly improve data efficiency. 
 An alternative representation of the state in the second period is of the form $s=(\text{website version } i, \text{checkout})$, retaining information about how the user navigated to the checkout. 
 In this case, there is no trajectory pooling and our theory indicates that TD behaves as MC would. 
By using a very rich representation, which recalls much of the past, the benefits of TD  disappear. 
Clearly, we want representations that are accurate, to avoid severe approximation errors. But we have shown that representations that are forgetful of aspects of the past offer enormous benefits; this lets  value-to-go from intermediate states serve as surrogate outcomes.
How should one balance this tradeoff while learning representations from data? This question closely relates to the one of understanding temporal consistency with function approximation.

\bibliography{ref}

\bibliographystyle{icml2022}

\newpage
\appendix
\onecolumn
\begin{appendix}

\section{Additional experiments}

Figure \ref{plot:var_horizons} and \ref{plot:meeting_horizons} aims at displaying the difference in precision of MC estimates and TD estimates on minimal working examples, to underline what is driving the difference. In this section, we complement the experiments in the main body to illustrate that the theory holds in more complex settings. Specifically, we show empirical results on an augmented version of the Layered MRP introduced in Figure \ref{graph:layered}, where we introduce cycles. Specifically, a cyclic Layered MRP with width $W$, horizon $T$ and backward arc probability $p$ has $T-1$ layers of $W$ states, each of which has a non-zero probability to each state in the next layer and with probability $p$, a transition to a state in a previous or the current layer. The structure is illustrated in Figure \ref{graph:layered_cycles}.
    
\begin{figure}[h!]
    \centering
\begin{tikzpicture}[main/.style = {draw, circle, minimum size=1cm}] 
\node[main, label=center:$s_1^{(1)}$] (11) {}; 
\node[main, label=center:$s_1^{(2)}$] (12) [below=0.3cm of 11] {}; 
\node (ellipsis1) [below=1.2cm of 11] {$\vdots$};
\node[main, label=center:$s_1^{(W)}$] (1w) [below=0.8cm of 12] {}; 

\node[main, label=center:$s^{(1)}_2$] (21) [right=1cm of 11] {}; 
\node[main, label=center:$s^{(2)}_2$] (22) [right=1cm  of 12] {}; 
\node (ellipsis2) [below=1.2cm of 21] {$\vdots$};
\node[main, label=center:$s^{(W)}_2$] (2w) [right=1cm  of 1w] {}; 



\node[main, label=center:$s^{(1)}_{T-1}$] (d1) [right=1.4cm of 21] {}; 
\node[main, label=center:$s^{(2)}_{T-1}$] (d2) [right=1.4cm of 22] {}; 
\node (ellipsisd) [below=1.2cm of d1] {$\vdots$};
\node[main, label=center:$s^{(W)}_{T-1}$] (dw) [right =1.4cm of 2w] {}; 

\node[main, label=center:$\emptyset$] (end) [right=0.6cm of d2] {};

\draw [->] (11) -- (21);
\draw [->] (12) -- (21);
\draw [->] (1w) -- (21);

\draw [->] (11) -- (22);
\draw [->] (12) -- (22);
\draw [->] (1w) -- (22);

\draw [->] (11) -- (2w);
\draw [->] (12) -- (2w);
\draw [->] (1w) -- (2w);

\draw [blue, ->] (2w) to [bend right = 1cm] (22);
\draw [blue, ->] (dw) to [bend right = 0.5cm] (21);







\draw [->] (d1) -- (end);
\draw [->] (d2) -- (end);
\draw [->] (dw) -- (end);

\draw[dashed] (22) -- (d2);

\end{tikzpicture}     

    \caption{\centering Layered MRP with width $W$ and horizon $T$ and backwards transitions. There is at most one backward transition per node. The probability of adding a backward transition is fixed $p$ and the desitination of a backward transition is chosen uniformly at random. Transitions are chosen randomly and rewards are uniform on $[r(s,s')-1;r(s,s')+1]$ where $r(s,s')$ is chosen uniformly between -1 and 1.}
    \label{graph:layered_cycles}
\end{figure}
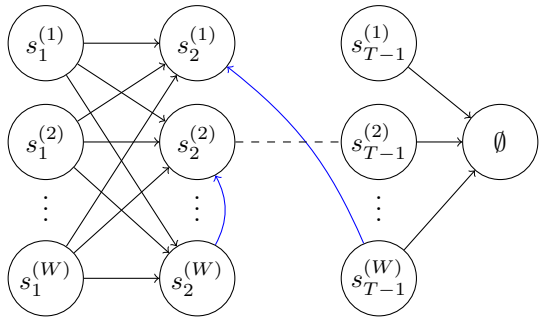

We first show that the Central Limit Theorems governing the MC and TD estimates still hold empirically when cycles are introduced. We used 10000 Monte-Carlo replications, each of which used 2000 samples to compute the value-to-go.

\newcommand{\MyDrawVarLinesCycle}{
\addplot [blue,mark=*,only marks, mark options={scale=.5}, error bars/.cd, y explicit,y dir=both] table[x=num_samples,y=Empirical_treatment_TD_A, y error plus expr=\thisrow{Empirical_treatment_TD_A_UB} - \thisrow{Empirical_treatment_TD_A} , y error minus expr=\thisrow{Empirical_treatment_TD_A} - \thisrow{Empirical_treatment_TD_A_LB}] {Data/var_horizons_p0.1.csv};
\addplot [blue] table[x=num_samples,y=Theoretical_treatment_TD_A] {Data/var_horizons_p0.1.csv};

\addplot [brown,mark=*,only marks, mark options={scale=.5}, error bars/.cd, y explicit,y dir=both] table[x=num_samples,y=Empirical_treatment_TD_B, y error plus expr=\thisrow{Empirical_treatment_TD_B_UB} - \thisrow{Empirical_treatment_TD_B} , y error minus expr=\thisrow{Empirical_treatment_TD_B} - \thisrow{Empirical_treatment_TD_B_LB}] {Data/var_horizons_p0.1.csv};
\addplot [brown] table[x=num_samples,y=Theoretical_treatment_TD_B] {Data/var_horizons_p0.1.csv};

\addplot [red,mark=*,only marks, mark options={scale=.5}, error bars/.cd, y explicit,y dir=both] table[x=num_samples,y=Empirical_treatment_TD_ATE, y error plus expr=\thisrow{Empirical_treatment_TD_ATE_UB} - \thisrow{Empirical_treatment_TD_ATE} , y error minus expr=\thisrow{Empirical_treatment_TD_ATE} - \thisrow{Empirical_treatment_TD_ATE_LB}] {Data/var_horizons_p0.1.csv};
\addplot [red] table[x=num_samples,y=Theoretical_treatment_TD_ATE] {Data/var_horizons_p0.1.csv};

\addplot [blue,mark=x,only marks, mark options={scale=1}, error bars/.cd, y explicit,y dir=both] table[x=num_samples,y=Empirical_treatment_MC_A, y error plus expr=\thisrow{Empirical_treatment_MC_A_UB} - \thisrow{Empirical_treatment_MC_A} , y error minus expr=\thisrow{Empirical_treatment_MC_A} - \thisrow{Empirical_treatment_MC_A_LB}] {Data/var_horizons_p0.1.csv};
\addplot [blue, dashed] table[x=num_samples,y=Theoretical_treatment_MC_A] {Data/var_horizons_p0.1.csv};

\addplot [brown,mark=x,only marks, mark options={scale=1}, error bars/.cd, y explicit,y dir=both] table[x=num_samples,y=Empirical_treatment_MC_B, y error plus expr=\thisrow{Empirical_treatment_MC_B_UB} - \thisrow{Empirical_treatment_MC_B} , y error minus expr=\thisrow{Empirical_treatment_MC_B} - \thisrow{Empirical_treatment_MC_B_LB}] {Data/var_horizons_p0.1.csv};
\addplot [brown, dashed] table[x=num_samples,y=Theoretical_treatment_MC_B] {Data/var_horizons_p0.1.csv};

\addplot [red,mark=x,only marks, mark options={scale=1}, error bars/.cd, y explicit,y dir=both] table[x=num_samples,y=Empirical_treatment_MC_ATE, y error plus expr=\thisrow{Empirical_treatment_MC_ATE_UB} - \thisrow{Empirical_treatment_MC_ATE} , y error minus expr=\thisrow{Empirical_treatment_MC_ATE} - \thisrow{Empirical_treatment_MC_ATE_LB}] {Data/var_horizons_p0.1.csv};
\addplot [red, dashed] table[x=num_samples,y=Theoretical_treatment_MC_ATE] {Data/var_horizons_p0.1.csv};
}
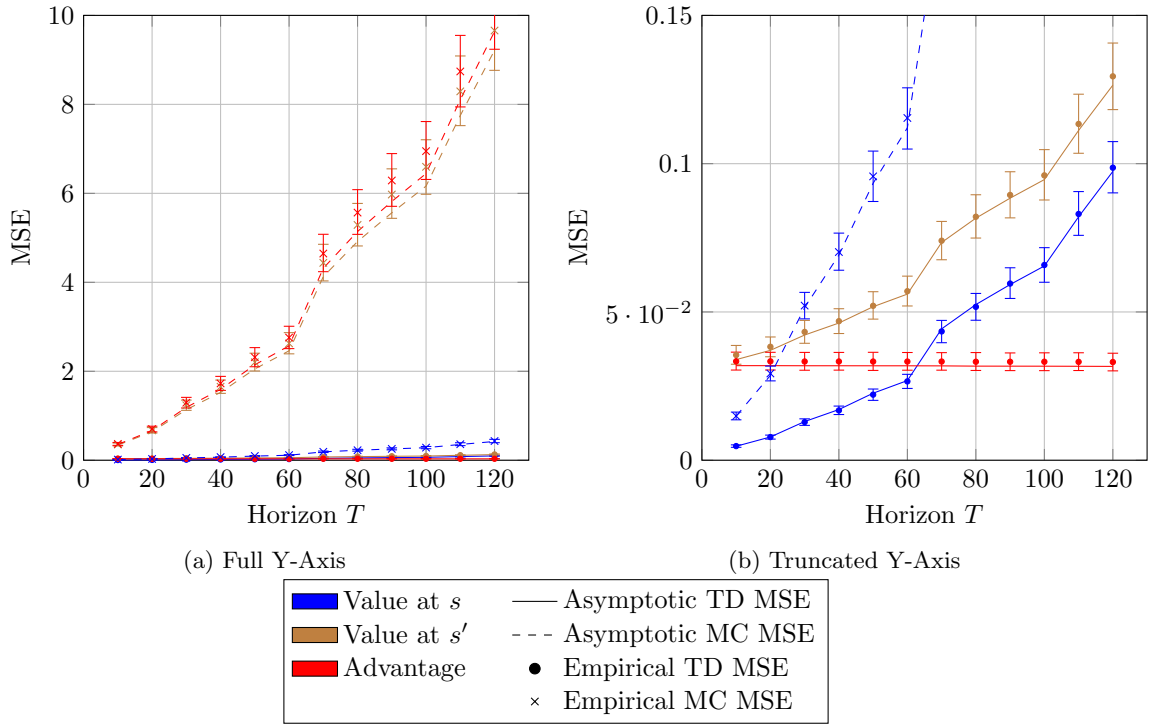
\begin{figure}[h!]
\begin{subfigure}[h]{0.49\columnwidth}
\centering    
\begin{tikzpicture}
\begin{axis}[
            xmin=0,xmax=130,
            ymin=0.0,ymax=10,
            width=7.5cm,
            height=7.5cm,
            table/col sep=comma,
            xlabel = Horizon $T$,
            ylabel = MSE,
            grid=both] 
    \MyDrawVarLinesCycle
\end{axis}
\end{tikzpicture}
\caption{Full Y-Axis}
\end{subfigure}
\begin{subfigure}[h]{0.49\columnwidth}
\centering
\begin{tikzpicture}
\begin{axis}[
            xmin=0,xmax=130,
            ymin=0.0,ymax=0.15,
            width=7.5cm,
            height=7.5cm,
            table/col sep=comma,
            xlabel = Horizon $T$,
            ylabel = MSE,
            grid=both,
            scaled y ticks=false,
            legend pos=north east,
            title style={at={(0.5,0)},anchor=north,yshift=-0.1}] 
    \MyDrawVarLinesCycle
\end{axis}
\end{tikzpicture}
\caption{Truncated Y-Axis}
\end{subfigure}
\centering
\begin{tikzpicture}
\begin{axis}[
            xmin=10,xmax=130,
            ymin=0.0,ymax=0.5,
            width=2cm,
            height=2cm,
            hide axis,
            table/col sep=comma,
            xlabel = Horizon $H$,
            ylabel = Variance,
            ylabel near ticks, yticklabel pos=right,
            grid=both,
            legend style={anchor=north west,
                legend cell align=left,
                /tikz/every even column/.append style={column sep=0.5cm}},
            legend columns = 2] 
\addlegendimage{fill = blue, area legend}
\addlegendentry{Value at $s$}
\addlegendimage{black}
\addlegendentry{Asymptotic TD MSE}
\addlegendimage{fill = brown, area legend}
\addlegendentry{Value at $s'$}
\addlegendimage{dashed,black}
\addlegendentry{Asymptotic MC MSE}

\addlegendimage{fill = red, area legend}
\addlegendentry{Advantage}
\addlegendimage{mark=*, only marks,black}
\addlegendentry{Empirical TD MSE}
\addlegendimage{opacity = 0}
\addlegendentry{}
\addlegendimage{mark=x, only marks,black}
\addlegendentry{Empirical MC MSE}

\end{axis}
\end{tikzpicture}

    \caption{\centering MSE of different MC and TD estimates on Cyclic Layered MRP with $W = 5$, $p=0.1$ and varying horizon $T$}
    \label{plot:var_horizons_p0.1}
\end{figure}

Next, we show how accurate the limiting approximation is in the finite sample regime. To do so, we compute the empirical MSE for the MC and TD estimates for a range of sample sizes and plot alongside the normal approximation. The MRP used for these experiment is a Layered MRP with cycles with parameter $W = 5, T = 120, p=0.1$. We used 10000 Monte-Carlo replications.

\newcommand{\MyDrawVarLinesSample}{
\addplot [blue,mark=*,only marks, mark options={scale=.5}, error bars/.cd, y explicit,y dir=both] table[x=num_samples,y=Empirical_treatment_TD_A, y error plus expr=\thisrow{Empirical_treatment_TD_A_UB} - \thisrow{Empirical_treatment_TD_A} , y error minus expr=\thisrow{Empirical_treatment_TD_A} - \thisrow{Empirical_treatment_TD_A_LB}] {Data/var_samples_p0.1.csv};
\addplot [blue] table[x=num_samples,y=Theoretical_treatment_TD_A] {Data/var_samples_p0.1.csv};

\addplot [brown,mark=*,only marks, mark options={scale=.5}, error bars/.cd, y explicit,y dir=both] table[x=num_samples,y=Empirical_treatment_TD_B, y error plus expr=\thisrow{Empirical_treatment_TD_B_UB} - \thisrow{Empirical_treatment_TD_B} , y error minus expr=\thisrow{Empirical_treatment_TD_B} - \thisrow{Empirical_treatment_TD_B_LB}] {Data/var_samples_p0.1.csv};
\addplot [brown] table[x=num_samples,y=Theoretical_treatment_TD_B] {Data/var_samples_p0.1.csv};

\addplot [red,mark=*,only marks, mark options={scale=.5}, error bars/.cd, y explicit,y dir=both] table[x=num_samples,y=Empirical_treatment_TD_ATE, y error plus expr=\thisrow{Empirical_treatment_TD_ATE_UB} - \thisrow{Empirical_treatment_TD_ATE} , y error minus expr=\thisrow{Empirical_treatment_TD_ATE} - \thisrow{Empirical_treatment_TD_ATE_LB}] {Data/var_samples_p0.1.csv};
\addplot [red] table[x=num_samples,y=Theoretical_treatment_TD_ATE] {Data/var_samples_p0.1.csv};

\addplot [blue,mark=x,only marks, mark options={scale=1}, error bars/.cd, y explicit,y dir=both] table[x=num_samples,y=Empirical_treatment_MC_A, y error plus expr=\thisrow{Empirical_treatment_MC_A_UB} - \thisrow{Empirical_treatment_MC_A} , y error minus expr=\thisrow{Empirical_treatment_MC_A} - \thisrow{Empirical_treatment_MC_A_LB}] {Data/var_samples_p0.1.csv};
\addplot [blue, dashed] table[x=num_samples,y=Theoretical_treatment_MC_A] {Data/var_samples_p0.1.csv};

\addplot [brown,mark=x,only marks, mark options={scale=1}, error bars/.cd, y explicit,y dir=both] table[x=num_samples,y=Empirical_treatment_MC_B, y error plus expr=\thisrow{Empirical_treatment_MC_B_UB} - \thisrow{Empirical_treatment_MC_B} , y error minus expr=\thisrow{Empirical_treatment_MC_B} - \thisrow{Empirical_treatment_MC_B_LB}] {Data/var_samples_p0.1.csv};
\addplot [brown, dashed] table[x=num_samples,y=Theoretical_treatment_MC_B] {Data/var_samples_p0.1.csv};

\addplot [red,mark=x,only marks, mark options={scale=1}, error bars/.cd, y explicit,y dir=both] table[x=num_samples,y=Empirical_treatment_MC_ATE, y error plus expr=\thisrow{Empirical_treatment_MC_ATE_UB} - \thisrow{Empirical_treatment_MC_ATE} , y error minus expr=\thisrow{Empirical_treatment_MC_ATE} - \thisrow{Empirical_treatment_MC_ATE_LB}] {Data/var_samples_p0.1.csv};
\addplot [red, dashed] table[x=num_samples,y=Theoretical_treatment_MC_ATE] {Data/var_samples_p0.1.csv};
}
\begin{figure}[h!]
\begin{subfigure}[h]{0.49\columnwidth}
\centering    
\begin{tikzpicture}
\begin{axis}[
            xmin=0,xmax=2200,
            ymin=0.0,ymax=110,
            width=7.5cm,
            height=7.5cm,
            table/col sep=comma,
            xlabel = Number of samples $n$,
            ylabel = MSE,
            grid=both] 
    \MyDrawVarLinesSample
\end{axis}
\end{tikzpicture}
\caption{Full Y-Axis}
\end{subfigure}
\begin{subfigure}[h]{0.49\columnwidth}
\centering
\begin{tikzpicture}
\begin{axis}[
            xmin=0,xmax=2200,
            ymin=0.0,ymax=1.4,
            width=7.5cm,
            height=7.5cm,
            table/col sep=comma,
            xlabel = Number of samples $n$,
            ylabel = MSE,
            grid=both,
            scaled y ticks=false,
            legend pos=north east,
            title style={at={(0.5,0)},anchor=north,yshift=-0.1}] 
            \MyDrawVarLinesSample
\end{axis}
\end{tikzpicture}
\caption{Truncated Y-Axis}
\end{subfigure}
\centering
\begin{tikzpicture}
\begin{axis}[
            xmin=10,xmax=130,
            ymin=0.0,ymax=0.5,
            width=2cm,
            height=2cm,
            hide axis,
            table/col sep=comma,
            xlabel = Horizon $H$,
            ylabel = Variance,
            ylabel near ticks, yticklabel pos=right,
            grid=both,
            legend style={anchor=north west,
                legend cell align=left,
                /tikz/every even column/.append style={column sep=0.5cm}},
            legend columns = 2] 
\addlegendimage{fill = blue, area legend}
\addlegendentry{Value at $s$}
\addlegendimage{black}
\addlegendentry{Asymptotic TD MSE}
\addlegendimage{fill = brown, area legend}
\addlegendentry{Value at $s'$}
\addlegendimage{dashed,black}
\addlegendentry{Asymptotic MC MSE}

\addlegendimage{fill = red, area legend}
\addlegendentry{Advantage}
\addlegendimage{mark=*, only marks,black}
\addlegendentry{Empirical TD MSE}
\addlegendimage{opacity = 0}
\addlegendentry{}
\addlegendimage{mark=x, only marks,black}
\addlegendentry{Empirical MC MSE}

\end{axis}
\end{tikzpicture}

    \caption{\centering MSE of MC and TD estimators when varying the number of samples, compared with the normal approximation}
    \label{plot:var_samples_p0.1}
\end{figure}
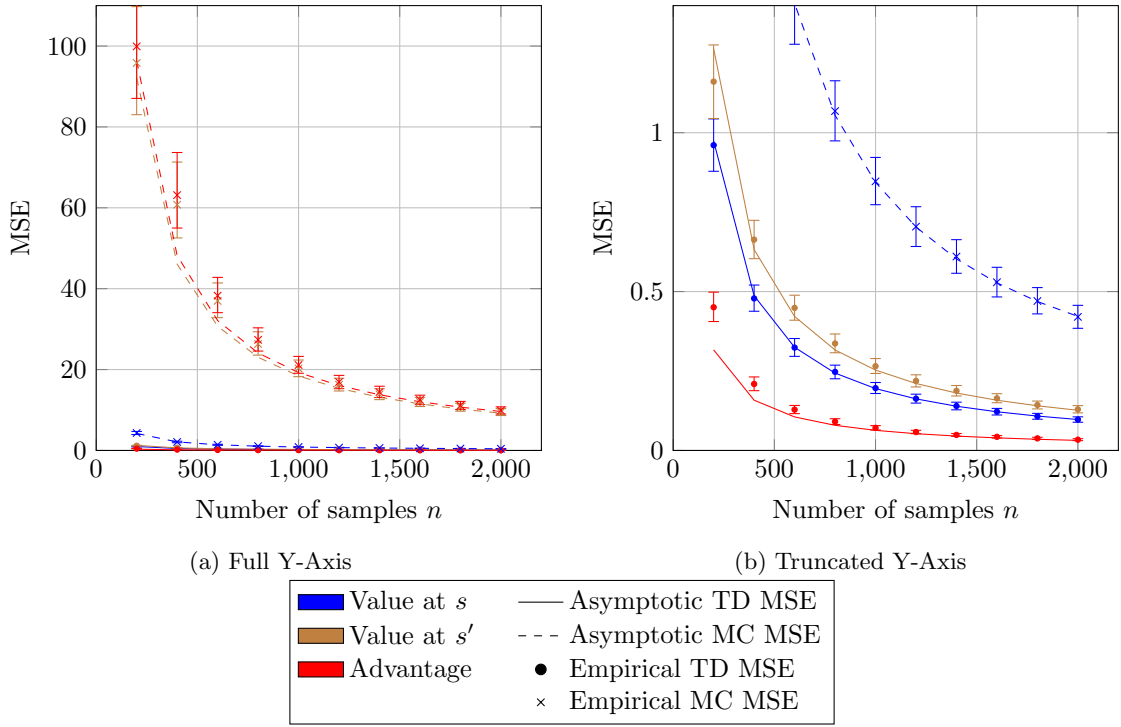

An alternate view on the impact of the number of samples is the one of regret: when trying to choose what state has the highest value function between two states $s$ and $s'$, how often would an estimator indicate the wrong decision. That is, the regret of an estimator $\hat{V}$ is
$\EE{\mathbbm{1}\left(\hat{V}(s) > \hat{V}(s')\right)}$, assuming $V(s') > V(s)$. We plot the empirical regret when using MC estimates versus TD estimates, as a function of the number of samples, along with the normal approximation derived from the Central Limit Theorems.

\begin{figure}[h!]
\centering
\begin{center}
\centering
\begin{subfigure}[h]{0.69\columnwidth}
\centering
\begin{tikzpicture}
\begin{axis}[
            title={},
            xmin=0,xmax=2200,
            ymin=0,ymax=0.5,
            width=7.5cm,
            height=7.5cm,
            table/col sep=comma,
            xlabel = Number of samples,
            ylabel = Regret,
            grid=both,]

\addplot [red,mark=*,only marks, mark options={scale=0.5}, error bars/.cd, y explicit,y dir=both] table[x=num_samples,y=Empirical_treatment_TD_ATE, y error plus expr=\thisrow{Empirical_treatment_TD_ATE_UB} - \thisrow{Empirical_treatment_TD_ATE} , y error minus expr=\thisrow{Empirical_treatment_TD_ATE} - \thisrow{Empirical_treatment_TD_ATE_LB}] {Data/regret.csv};
\addplot [red] table[x=num_samples,y=Theoretical_treatment_TD_ATE] {Data/regret.csv};

\addplot [blue,mark=*,only marks, mark options={scale=1}, error bars/.cd, y explicit,y dir=both] table[x=num_samples,y=Empirical_treatment_MC_ATE, y error plus expr=\thisrow{Empirical_treatment_MC_ATE_UB} - \thisrow{Empirical_treatment_MC_ATE} , y error minus expr=\thisrow{Empirical_treatment_MC_ATE} - \thisrow{Empirical_treatment_MC_ATE_LB}] {Data/regret.csv};
\addplot [blue] table[x=num_samples,y=Theoretical_treatment_MC_ATE] {Data/regret.csv};

\end{axis}
\end{tikzpicture}
\end{subfigure}
\begin{subfigure}[h]{0.2\columnwidth}
\begin{tikzpicture}
\begin{axis}[
            xmin=10,xmax=130,
            ymin=0.0,ymax=0.5,
            width=2cm,
            height=2cm,
            hide axis,
            table/col sep=comma,
            xlabel = Horizon $H$,
            ylabel = Variance,
            ylabel near ticks, yticklabel pos=right,
            grid=both,
            legend style={anchor=north west,
                legend cell align=left,
                /tikz/every even column/.append style={column sep=0.5cm}}]  
\addlegendimage{fill = blue, area legend}
\addlegendentry{MC estimate}

\addlegendimage{fill = red, area legend}
\addlegendentry{TD estimates}
\addlegendimage{black}
\addlegendentry{Normal approximation}
\addlegendimage{mark=*, only marks,black}
\addlegendentry{Empirical estimate}

\end{axis}
\end{tikzpicture}
\end{subfigure}
\end{center}
\caption{Regret for TD and MC estimates on Cyclic Layered MRP with $W =5,p=0.1, T=120$ as the number of samples used vary.}
\end{figure}
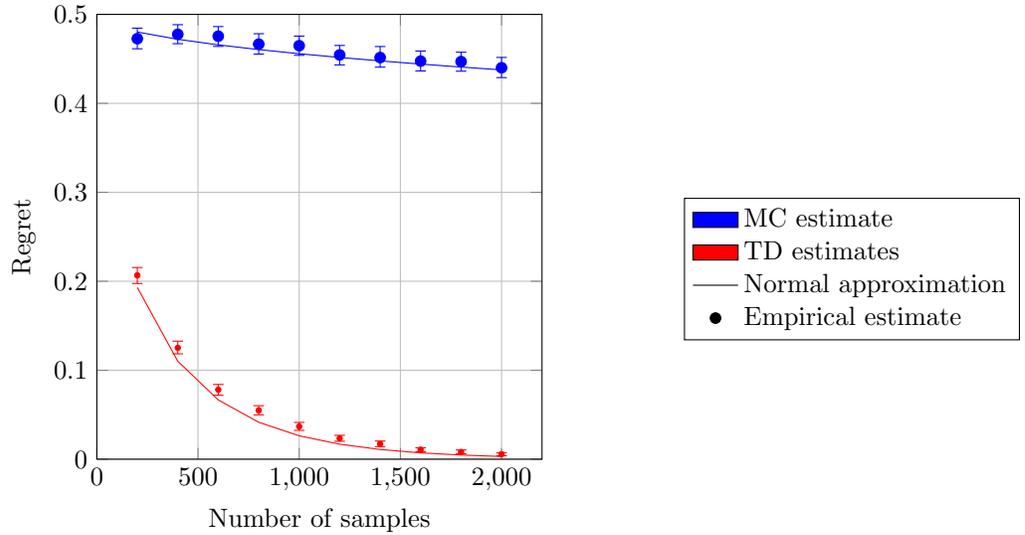

\section{Proofs.}

We state and prove results in the context of weighted value function which is a linear combination of the value function evaluated at individual states. 
\begin{definition}[Weighted value function]
	For a weighting over states $\pi$, the extended value function is defined as 
	\begin{equation*}
		J(\pi) = \sum_{s \in \mathcal{S} } \pi(s) V(s)
	\end{equation*}
\end{definition}

By setting $\pi$ to be a mass point at a single state, we recover the value function at this state. When interpreting initial states as actions (as in Section \ref{subsec:HTE}), we recover randomized policy when using a distribution over actions as the weighting. In this case, the weighted value function is the expected value when playing according to the randomized policy. Note that our definition of weighted value function allows for any weighting of states, including negative weights. This will be useful for analyzing advantages $V(s) - V(s')$ by setting $\pi(s) = 1$ and $\pi(s') = -1$. 

 We also extend our definition of expected number of visits to weightings over initial states.

\begin{definition}[Weighted expected number of visits]
   For a weighting over states $\pi$, we write $\eta_{\pi}(s)$ the weighted number of visits to $s$:
   \begin{equation*}
       \eta_{\pi}(s) =  \sum_{s' \in \mathcal{S}}  \pi(s') \EE{N(s) |S_0 = s'}
   \end{equation*}
\end{definition}

Similarly to the weighted value function, $\pi$ is not enforced to be a distribution over state, allowing even for negative values. In the case where $\pi$ is a distribution over state, we recover the probabilistic interpretation: $\eta_\pi(s)$ is the expected number of visits to state $s$ when the initial distribution is $\pi$.

\begin{definition}[One-step variance]
	\begin{equation*}
		\OSV{s} = \Var{R_t + V(S_{t+1}) \mid S_t=s}.
	\end{equation*}
\end{definition}

We extend trajectories into infinite horizon trajectories that stay in the terminating state and stop collecting rewards once the terminating state is reached: $S_t = \emptyset$ and $R_{t+1} = 0$ for all $t \geq T$. Equivalently, we define the transition $P(\emptyset \mid \emptyset) = 1$ and the reward $R(\emptyset, \emptyset) = 0$ a.s.

We start by stating and proving Central Limit Theorems (CLT) for the convergence of both TD and MC estimates. We then use these two results as building blocks to prove the main theorems. 
\subsection{Central Limit Theorems}

\begin{proposition}[Central Limit Theorem for  MC]
\label{lem:MC_variance}
For $s \in \mathcal{S}$,
\begin{equation*}
        \sqrt{n}(\hat{V}_{\rm MC}(s) - V(s)) \Rightarrow \mathcal{N}\left(0,  \dfrac{1}{\pr{s\in \tau}}\sum_{s' \in \mathcal{S}}\EE{N(s') \mid S_0 = s} \OSV{s'}\right)
\end{equation*}

\end{proposition}

\begin{proof}

We recall that, for tabular representation, the MC estimator takes the form

\begin{align*}
	\hat{V}_{\rm MC}(s) &= \mathbb{E}_{(S, V) \sim\mathcal{D}^{\rm MC}}\left[ V \mid S=s \right]\\
					&= \dfrac{1}{\mid I(s) \mid}\sum_{i\in I(s)}\sum_{t=T^{(i)}(s)}^{T^{(i)}} R_{t+1}^{(i)},
\end{align*}

where $I(s)$ is the set of trajectories that visit state $s$ and $T^{(i)}(s)$ is the first visit to state $s$ in trajectory $i$. Since we consider first-visit MC, each trajectory appears at most once in the summation. We start by rewriting the error $\hat{V}_{\rm MC}(s) - V(s)$ as the product of a scaling factor and the average of i.i.d. random variables:

\begin{align*}
    \hat{V}_{\rm MC}(s) - V(s) &= \dfrac{1}{\mid I(s) \mid}\sum_{i\in I(s)}\left(\sum_{t=T^{(i)}(s)}^{T^{(i)}} R_{t+1}^{(i)} - V(s)\right)\\
    &= \dfrac{n}{\mid I(s) \mid} \cdot \dfrac{1}{n}\sum_{i = 1}^n \left(\mathbbm{1}(s \in \tau^{(i)})\left(\sum_{t=T^{(i)}(s)}^{T^{(i)}} R_{t+1}^{(i)} - V(s)\right)\right)
\end{align*}

Recall that if $s$ is not visited in trajectory $i$, $T^{(i)}(s)$ is defined to be $T^{(i)}(s) = T^{(i)}$. 
\begin{itemize}
    \item We start by proving a Central Limit Theorem on 
    \begin{equation*}
         \dfrac{1}{n}\sum_{i = 1}^n \left(\mathbbm{1}(s \in \tau^{(i)})\sum_{t=T^{(i)}(s)}^{T^{(i)}} R_{t+1}^{(i)}\right).
    \end{equation*}

    The variables $\left(\mathbbm{1}(s \in \tau^{(i)})\left(\sum_{t=T^{(i)}(s)}^{T^{(i)}} R_{t+1}^{(i)} - V(s)\right)\right)_{i=1,\dots, n}$ are $n$ i.i.d., zero mean random variables. We now compute their variance.

    \begin{align*}
        \Var{\mathbbm{1}(s \in \tau)\left(\sum_{t=T(s)}^{T} R_{t+1} - V(s)\right)} &= \pr{s\in \tau}\Var{\mathbbm{1}(s \in \tau)\left(\sum_{t=T(s)}^{T} R_{t+1} - V(s)\right) \mid s \in \tau}\\
        &= \pr{s\in \tau}\Var{\sum_{t=T(s)}^{T} R_{t+1} - V(s) \mid s \in \tau}
    \end{align*}
    
Since the summation starts at the stopping time defined by the first visit to state $s$, the Strong Markov Property enables to re-index the summation in the following way:

    \begin{align*}
        \Var{\mathbbm{1}(s \in \tau)\left(\sum_{t=T(s)}^{T} R_{t+1} - V(s)\right)} &= \pr{s\in \tau}\Var{\sum_{t=0}^{T} R_{t+1} - V(s) \mid S_0 = s}\\
         &= \pr{s\in \tau}\Var{\sum_{t=0}^{\infty} R_{t+1} - V(s) \mid S_0 = s}
    \end{align*}
where we allowed the sum to run to infinity since $(R_{t+1})_t$ is a.s. stationary at 0 for $t \geq T$. Similarly, we use the fact that $\left(V(S_t) - V(S_{t+1})\right)$ is a.s. stationary at 0 to write $V(S_0) = \sum_{t=0}^{\infty} \left(V(S_t) - V(S_{t+1})\right)$. Plugging in the previous expression gives:

\begin{align*}
     \Var{\mathbbm{1}(s \in \tau)\left(\sum_{t=T(s)}^{T} R_{t+1} - V(s)\right)} &= \pr{s\in \tau}\Var{\sum_{t=0}^\infty \left(R_{t+1} + V(S_{t+1}) - V(S_t)\right) | S_0 = s}\\
\end{align*}

Notice that $\left(R_{t+1} +  V(S_{t+1}) - V(S_t)\right)_t$ are martingale differences with respect to the filtration $\mathcal{F}_t = \{S_0, \dots S_t\}$. Using that martingale differences are uncorrelated:

\begin{align*}
    \Var{\mathbbm{1}(s \in \tau)\left(\sum_{t=T(s)}^{T} R_{t+1} - V(s)\right)} &=\pr{s\in \tau}\sum_{t=0}^\infty\EE{(V(S_{t+1}) + R_{t+1} -  V(S_{t}))^2 |S_0 = s} \\
\end{align*}

We then group the terms in the sum by the value of $S_t$:
\begin{align*}
     &\Var{\mathbbm{1}(s \in \tau)\left(\sum_{t=T(s)}^{T} R_{t+1} - V(s)\right)} = \\&\quad \quad \quad\quad \quad \quad\pr{s\in \tau}\sum_{t=0}^\infty\sum_{s' \in \mathcal{S}}\pr{S_t = s' | S_0 = s}\EE{(V(S_{t+1}) + R_{t+1} -  V(S_{t}))^2 |S_t = s'}\\
    &\quad \quad \quad\quad \quad \quad= \pr{s\in \tau}\sum_{t=0}^\infty\sum_{s' \in \mathcal{S}}\pr{S_t = s' | S_0 = s} \OSV{s'} \\
    &\quad \quad \quad\quad \quad \quad= \pr{s\in \tau}\sum_{s' \in \mathcal{S}} \OSV{s'}\sum_{t=0}^{\infty}\pr{S_t = s' | S_0 = s}\\
    &\quad \quad \quad\quad \quad \quad= \pr{s\in \tau}\sum_{s' \in \mathcal{S}} \EE{N(s') | S_0 = s} \OSV{s'}
\end{align*}

Using the Central Limit Theorem, we obtain the following convergence:

    \begin{equation*}
         \sqrt{n}\dfrac{1}{n}\sum_{i = 1}^n \left(\mathbbm{1}(s \in \tau^{(i)})\sum_{t=T^{(i)}(s)}^{T^{(i)}} R_{t+1}^{(i)}\right)\Rightarrow \mathcal{N}\left(0, \pr{s\in \tau}\sum_{s' \in \mathcal{S}} \EE{N(s') | S_0 = s} \OSV{s'}\right).
    \end{equation*}

\item The Strong Law of Large Number ensures:
\begin{equation*}
    \dfrac{n}{\mid I(s) \mid} \underset{n \to \infty}{\longrightarrow} \dfrac{1}{\pr{s\in \tau}} \quad \textrm{a.s.}.
\end{equation*}
\end{itemize}
 Finally, using Slutsky's Theorem, the product converges: 
 \begin{equation*}
     \dfrac{n}{\mid I(s) \mid} \cdot \dfrac{1}{n}\sum_{i = 1}^n \left(\mathbbm{1}(s \in \tau^{(i)})\left(\sum_{t=T^{(i)}(s)}^{T^{(i)}} R_{t+1}^{(i)} - V(s)\right)\right) \Rightarrow \mathcal{N}\left(0, \dfrac{1}{\pr{s\in \tau}}\sum_{s' \in \mathcal{S}} \EE{N(s') | S_0 = s} \OSV{s'}\right).
 \end{equation*}
\end{proof}

\begin{proposition}[Central Limit Theorem for TD]
\label{lem:TD_variance}
For any weighting $\pi$,
 \begin{equation*}
    \sqrt{n}(\hat{J}_{\rm TD}(\pi) - J(\pi)) \Rightarrow \mathcal{N}\left(0, \sum_{s' \in \mathcal{S}} \dfrac{\eta^2_\pi(s') \OSV{s'}}{\EE{N(s')}} \right).
\end{equation*}
\end{proposition}

\begin{corollary}
For $s \in \mathcal{S}$
\begin{equation*}
    \sqrt{n}(\hat{V}_{TD}(s) - V(s)) \Rightarrow \mathcal{N}\left(0, \sum_{s' \in \mathcal{S}} \dfrac{\EE{N(s') \mid S_0 = s}^2 \OSV{s'}}{\EE{N(s')}} \right).
\end{equation*}
\end{corollary}

\subsection{Proof of Proposition \ref{lem:TD_variance}}
\paragraph{Notations}
The idea of the proof is to use the fixed point identities verified by $V$ and $\hat{V}_{\rm TD}$ to obtain a recursive formula for the error $\hat{V}_{\rm TD} - V$. We then analyze the two quantities that appear when iterating this recursive formula: one is a sum of empirical one-step error and we show the other is of second order. 

We adopt a vectorial representation of the value function $V = \left(V(s)\right)_{s \in \mathcal{S}}$ such that $J(\pi) = \langle V, \pi \rangle$. 
Let $T$ be the Bellman operator:

\begin{equation*}
	\mathcal{T}(f)(s) = \EE[S' \sim P(\cdot | s)] {R(s,S') + f(S')} 
\end{equation*}

The value function is the unique fixed point of the Bellman operator: $V = T(V)$ to verify $V(\emptyset) = 0$. We also write the transition operator as follow:
\begin{equation*}
	P(f)(s) =  \EE[S' \sim P(\cdot | s)] {f(S')}  = \langle P(\cdot \mid s), f \rangle.
\end{equation*}

As discussed in Section \ref{sec:TD_learning}, by trying to minimize temporal differences, TD solves the \emph{empirical Bellman equation}:

\begin{equation*}
	\hat{V}_{\rm TD}(s) = \mathbb{E}_{(S, R,S')\sim \mathcal{D}^{\rm TD}}\left[ R+\hat{V}^{\rm TD}(S') \mid S=s \right].
\end{equation*}

Solving the \emph{empirical Bellman equation} can be viewed as being a fixed point (with $f(\emptyset) = 0$) of the empirical Bellman operator that we define as follow: 

\begin{equation*}
	\hat{\mathcal{T}}(f)(s) = \mathbb{E}_{(S, R,S')\sim \mathcal{D}^{\rm TD}}\left[ R+f(S') \mid S=s \right].
\end{equation*}

Similarly, we note the empirical transition operator 
\begin{equation*}
	\hat{P}(f)(s) =  \mathbb{E}_{(S, R,S')\sim \mathcal{D}^{\rm TD}}\left[f(S') \mid S=s \right].
\end{equation*}

Finally, we introduce an explicit notation of the empirical Bellman operator that will ease proofs. To do so, we first need to introduce $B(s)$, the set of all visits to state $s$: 
\begin{equation*}
    B(s) = \{ (i,t) \mid S_t^{(i)} = i\}
\end{equation*}

Using this notation, the empirical Bellman operator can be written as follow:
\begin{equation*}
	\hat{\mathcal{T}}(f)(s_0) = \dfrac{1}{\mid B(s_0) \mid } \sum_{(i,t) \in B(s_0)} \left( f\left(S_{t+1}^{(i)}\right) + R^{(i)}_{t+1} \right).
\end{equation*}

\paragraph{Analysis}
We start by expanding the error vector $\hat{V}_{\rm TD} - V$ into a sum of empirical one-step errors and a second term that we later show to be of second order. 
\begin{lemma}
    \label{lem:taylor_exp}
    \begin{align*}
        \sqrt{n}\left(\hat{V}_{\rm TD} - V\right) &= \sqrt{n}\sum_{t=0}^{\infty} P^t(\hat{\mathcal{T}}V-V) + \sum_{t=0}^{\infty} P^t \left( \sqrt{n}(\hat{P}-P)(\hat{V}_{\rm TD} - V)\right)  \\
    \end{align*}
\end{lemma}
Note that, since $P^t(S_0) = S_t$ is stationary at $\emptyset$ almost surely, the quantities summed in Lemma \ref{lem:taylor_exp} are ultimately zero almost surely. 
\begin{proof}
    
By expending the difference $\hat{V}_{\rm TD} - V$ using the fixed point identities, we obtain a recursive identity:

\begin{align*}
	\hat{V}_{\rm TD} - V &= \hat{\mathcal{T}}\hat{V}_{\rm TD} - \mathcal{T}V \\
	&=  (\hat{\mathcal{T}}-\mathcal{T})\hat{V}_{\rm TD} + \mathcal{T}\hat{V}_{\rm TD} - \mathcal{T}V\\
    &= \left((\hat{\mathcal{T}}-\mathcal{T})\hat{V}_{\rm TD} - (\hat{\mathcal{T}}-\mathcal{T})V\right) + (\hat{\mathcal{T}}-\mathcal{T})V + \left( \mathcal{T}\hat{V}_{\rm TD} - \mathcal{T}V\right)\\
	&\stackrel{(a)}{=}  (\hat{P}-P)(\hat{V}_{\rm TD} - V) +  (\hat{\mathcal{T}}-\mathcal{T})V + P(\hat{V}_{\rm TD} - V)\\
	&=  (\hat{P}-P)(\hat{V}_{\rm TD} - V) +  (\hat{\mathcal{T}}V-V) + P(\hat{V}_{\rm TD} - V)
\end{align*}
where (a) holds because $\mathcal{T}f = Pf + R$. Hence $TV - T\hat{V}_{\rm TD} = P(V-\hat{V}_{\rm TD})$. This is a propriety of affine operator and can also be applied to the affine operator $(\hat{\mathcal{T}}-\mathcal{T})$.

Iterating this identity and multiplying by $\sqrt{n}$ on both sides gives

\begin{equation}
\label{eq:taylor_decomposition}
\sqrt{n}\left(\hat{V}_{\rm TD} - V\right) = \sum_{t=0}^{\infty} P^t \left( \sqrt{n}(\hat{P}-P)(\hat{V}_{\rm TD} - V)\right) +  \sqrt{n}\sum_{t=0}^{\infty} P^t(\hat{\mathcal{T}}V-V).
\end{equation}
\end{proof}

We now state two lemmas that analyze the two terms that appear in equation (\ref{eq:taylor_decomposition}). They are proved later. 

\begin{lemma}
\label{lem:second_order}
For all $s \in \mathcal{S}$, as $n\to \infty$,
    \begin{equation*}
    \sum_{t=0}^{\infty} P^t \left( \sqrt{n}(\hat{P}-P)(\hat{V}_{\rm TD} - V)\right)(s) \to 0 \quad \textrm{a.s.}
\end{equation*}
\end{lemma}

\begin{lemma} 
\label{lem:leading_term}
As $n\to \infty$
\begin{equation*}
    	\sqrt{n}\sum_{s \in \mathcal{S}} \pi(s) \sum_{t=0}^{\infty} P^t(\hat{\mathcal{T}}V-V)(s) \Rightarrow \mathcal{N}\left( 0, \sum_{s \in \mathcal{S}} \frac{\eta_\pi(s)^2}{\EE{N(s)}}\OSV{s}\right)
\end{equation*}
\end{lemma}

Finally, combining Lemma \ref{lem:second_order} and Lemma \ref{lem:leading_term} with Slutstky Lemma concludes the proof.

We now proceed to prove Lemma \ref{lem:second_order} and Lemma \ref{lem:leading_term}, starting with Lemma \ref{lem:second_order}

\subsubsection{Proof of Lemma \ref{lem:second_order}}

\begin{proof}

We start by showing that, after appropriate rescaling the error in transition estimates converge to a Gaussian distribution. 
We do not calculate the variance here since it is not needed. 

\begin{lemma}
    \label{lem:CLT_transitions}
    For each $s$,
    $\sqrt{n}\left(\hat{P} - P \right)(\cdot | s)$ weakly converges to a normal distribution with mean zero.
\end{lemma}

\begin{proof}
    We first express $\hat{P}( s' | s) - P(s' | s)$ as the product of a scaling factor and the average of $n$ i.i.d random vectors: 
    \begin{align*}
        \left(\hat{P}(s' | s) - P(s' |s)\right)_{s' \in \mathcal{S}} &= \dfrac{1}{\mid B(s) \mid}\sum_{(i,t) \in B(s)} \left( \mathbbm{1}(S_{t+1}^{(i)} = s') - P(s'|s)\right)_{s' \in \mathcal{S}}\\
        &= \dfrac{n}{\mid B(s) \mid} \cdot \dfrac{1}{n} \sum_{i = 1}^n \left(\sum_{t=0}^{\infty} \mathbbm{1}(S_{t}^{(i)} = s)\left(\mathbbm{1}(S_{t+1}^{(i)} = s') - P(s' |s)\right)\right)_{s' \in \mathcal{S}.}\\
    \end{align*}
    We decompose this identity in two terms:

    \begin{itemize}
        \item Strong law of large numbers ensures that 
        \begin{equation*}
            \dfrac{n}{\mid B(s) \mid} \underset{n \to \infty}{\longrightarrow} \dfrac{1}{\EE{N(s)}} \quad \textrm{a.s.}.
        \end{equation*}
        \item Since trajectories are independent, 
        \begin{equation*}
            \sum_{t=0}^{\infty} \mathbbm{1}(S_{t}^{(i)} = s)\left(\mathbbm{1}(S_{t+1}^{(i)} = s') - P(s'|s)\right)_{s' \in \mathcal{S}}
        \end{equation*}
        are independent, identically distributed random variables with finite variance and mean 0. The Central Limit Theorem ensures that the rescaled average of these random variables
        \begin{equation*}
         \sqrt{n}\cdot\dfrac{1}{n} \sum_{i = 1}^n \sum_{t=0}^{\infty} \mathbbm{1}(S_{t}^{(i)} = s)\left(\mathbbm{1}(S_{t+1}^{(i)} = s') - P(s'|s)\right)_{s' \in \mathcal{S}}
        \end{equation*}
        weakly converges to a mean 0 normal distribution. 
    \end{itemize}
    Finally, the product of these two quantities also converges to a mean 0 normal distribution. 
\end{proof}

We terminate the proof of Lemma \ref{lem:second_order} by combining Lemma \ref{lem:CLT_transitions} with the fact that $\hat{V}_{\rm TD} - V$ converges almost surely to 0 by Slutsky's Theorem. 

\end{proof}

\subsubsection{Proof of Lemma \ref{lem:leading_term}}
\begin{proof}
We start by expanding the transition operator $P^t$, 

\begin{align*}
	\sqrt{n}\sum_{t=0}^{\infty} P^t(\hat{\mathcal{T}}V-V)(s)  &= \sqrt{n} \sum_{t=0}^{\infty} \EE{\hat{\mathcal{T}}V(S_t) - V(S_t) \mid S_0 = s, \hat{\mathcal{T}}}\\
\end{align*} 

This term consists of a sum one-step differences of the form $\sqrt{n}\left(\hat{\mathcal{T}}V(S_t^{(i)}) - V(S_t^{(i)})\right)$. Grouping $S_t^{(i)}$ that are equal:

\begin{align*}
	\sqrt{n}\sum_{t=0}^{\infty} P^t(\hat{\mathcal{T}}V-V)(s)  &= \sqrt{n} \cdot \sum_{t=0}^{\infty} \EE{\sum_{s'\in \mathcal{S}} \mathbbm{1}(S_t = s') \left(\hat{\mathcal{T}}V(s') - V(s') \right)\mid S_0 = s, \hat{\mathcal{T}}}\\
	&= \sqrt{n} \cdot \EE{ \sum_{s'\in \mathcal{S}} \sum_{t=0}^{\infty} \mathbbm{1}(S_t = s') \left(\hat{\mathcal{T}}V(s') - V(s') \right)\mid S_0 = s, \hat{\mathcal{T}}}\\
	&= \sqrt{n}\sum_{s'\in \mathcal{S}}  \EE{ \sum_{t=0}^{\infty} \mathbbm{1}(S_t = s') \mid S_0 = s'} \left(\hat{\mathcal{T}}V(s') - V(s') \right)\\
	&= \sqrt{n}\sum_{s'\in \mathcal{S}}  \EE{N(s') | S_0 = s} \left(\hat{\mathcal{T}}V(s') - V(s') \right).\\
\end{align*} 

To get the result for a general weighting of states $\pi$, we take the dot product of the previous identity with $\pi$:

\begin{equation}
	\label{eq:TD_sum_part}
	\sqrt{n}\sum_{s \in \mathcal{S}} \pi(s) \sum_{t=0}^{\infty} P^t(\hat{\mathcal{T}}V-V)(s) = \sqrt{n}\sum_{s\in \mathcal{S}}  \eta_{\pi}(s) \left(\hat{T}V(s) - V(s) \right).
\end{equation}

We are left with analyzing a linear combination of terms of the form $\hat{\mathcal{T}}V(s) - V(s)$. For an individual $s$, $\hat{\mathcal{T}}V(s) - V(s)$ is an average of i.i.d. variables and its asymptotical behavior can be controlled using a Central Limit Theorem. However, to find the limiting distribution of the linear combination, we need to prove a vectorial Central Limit Theorem for the random vector $\left(\hat{\mathcal{T}}V(s) - V(s)\right)_{s\in \mathcal{S}}$. In particular, we show that $\left(\hat{\mathcal{T}}V(s) - V(s)\right)_{s\in \mathcal{S}}$ converge to independent normal distributions. 

The following lemma is key to prove the independence of the limiting distribution. It states that one-step differences are uncorrelated:
\begin{lemma}
   \label{lem:uncorrelated}

For $(s,t)  \neq (s',t')$: 

\begin{equation*}
    \Cov{\mathbbm{1}(S_t = s)\left(V(S_t) - V(S_{t+1}) - R_{t+1}\right), \mathbbm{1}(S_{t'} = s')\left(V(S_{t'}) - V(S_{t'+1}) - R_{t'+1}\right)} = 0.
\end{equation*}
\end{lemma}

\begin{proof}
The result follows from the Markovian property:
\begin{align*}
    &\Cov{\mathbbm{1}(S_t = s)\left(V(S_t) - V(S_{t+1}) - R_{t+1}\right), \mathbbm{1}(S_{t'} = \tilde{s})\left(V(S_{t'}) - V(S_{t'+1}) - R_{t' + 1}\right) } \\
    & \quad\quad  \quad  = \EE{ \Cov{\mathbbm{1}(S_t = s)\left(V(S_t) - V(S_{t+1}) - R_{t+1}\right), \mathbbm{1}(S_{t'} = s')\left(V(S_{t'}) - V(S_{t'+1}) - R_{t' + 1}\right) \mid S_t, S_{t+1}, S_{t'}} }\\
    & \quad\quad  \quad  = \EE { \mathbbm{1}(S_t = s)\left(V(S_t) - V(S_{t+1}) - R_{t+1}\right) \mathbbm{1}(S_{t'} = s')\left(V(S_{t'}) - \EE{V(S_{t'+1}) + R_{t'+1} | S_{t'}}\right)} \\
    & \quad \quad  \quad = 0.
\end{align*}
\end{proof}

Given the previous result, we can now state and proof the joint Central Limit Theorem of empirical one-step differences:

\begin{lemma}
    \label{lem:OS_CLT} 
    As $n\to \infty$,
        \begin{equation*}
        \sqrt{n}\left(\left(\hat{T}V(s) - V(s) \right)\right)_{s \in \mathcal{S}} \Rightarrow \mathcal{N}\left(0, \Sigma \right)
    \end{equation*}
    where $\Sigma$ is a diagonal matrix with $$\Sigma_{s,s} = \frac{1}{\EE{N(s)}}\sigma_V^2(s).$$
\end{lemma}
\begin{proof}
    We use a similar approach as in the proof of Lemma \ref{lem:CLT_transitions}: we express $\hat{\mathcal{T}}V(s) - V(s)$ as the product of a scaling factor that converges almost surely and an average of i.i.d. random vectors that is controlled by the Central Limit Theorem:
    \begin{align*}
        \sqrt{n}\cdot\left(\hat{\mathcal{T}}V(s) - V(s)\right) &= \sqrt{n} \cdot \dfrac{1}{\mid B(s) \mid} \sum_{(i,t) \in B(s)} (V(S_{t+1}^{(i)}) + R_{t+1}^{(i)} - V(s))\\
        &= \sqrt{n}\cdot\dfrac{1}{\mid B(s) \mid} \sum_{i = 1}^{n} \sum_{t=0}^{\infty} \mathbbm{1}(S_t^{(i)} = s)(V(S_{t+1}^{(i)}) + R_{t+1}^{(i)} - V(S_t^{(i)}))\\
        &= \dfrac{n}{\mid B(s) \mid} \left(\sqrt{n}\cdot\dfrac{1}{n}\sum_{i = 1}^{\infty} \sum_{t=0}^{\infty} \mathbbm{1}(S_t^{(i)} = s)(V(S_{t+1}^{(i)}) + R_{t+1}^{(i)} - V(S_t^{(i)}))\right)
    \end{align*}

\begin{itemize}
    \item $B(s)$ can be rewritten as the sum of i.i.d. random variables:
        \begin{equation*}
            B(s) = \sum_{i=1}^n \left(\sum_{t=0}^\infty \mathbbm{1}(S_t^{(i)} = s) \right).
        \end{equation*}
    The average $\dfrac{\mid B(s) \mid}{n}$ converges almost surely to $\EE{N(s)}$ by the Strong Law of Large Numbers. Taking the inverse gives:
    \begin{equation}
        \label{eq:scale_SLN}
        \dfrac{n}{\mid B(s) \mid} \to \dfrac{1}{\EE{N(s)}} \quad \textrm{a.s.}
    \end{equation}
    \item The random vector
        \begin{equation*}
            \sqrt{n} \cdot \dfrac{1}{n}\sum_{i = 1}^{n} \left(\sum_{t=0}^{\infty} \mathbbm{1}(S_t^{(i)} = s)(V(S_{t+1}^{(i)}) + R_{t+1}^{(i)} - V(S_t^{(i)}))\right)_{s \in \mathcal{S}}
        \end{equation*}
    is the re-scaled average of $n$ i.i.d., mean 0, vectors of the form $\left(\sum_{t=0}^{\infty} \mathbbm{1}(S_t = s)(V(S_{t+1}) + R_{t+1} - V(S_t))\right)_{s \in \mathcal{S}}$. The Central Limit Theorem ensures that this quantity converges to a normal distribution with mean 0. We now proceed to compute the variance of this distribution. 
    \begin{itemize}
        \item Lemma \ref{lem:uncorrelated} ensures that entries of this vector are uncorrelated:

        \item  The variance of a single entry is given by 
 \begin{align*}
     \Var{\sum_{t=0}^{\infty} \mathbbm{1}(S_t = s)(V(S_{t+1}) + R_{t+1} - V(S_t))} &\stackrel{(a)}{=} \sum_{t=0}^{\infty} \Var{\mathbbm{1}(S_t = s)(V(S_{t+1}) + R_{t+1} - V(S_t))} \\
     &= \sum_{t=0}^{\infty} \EE{\mathbbm{1}(S_t = s)(V(S_{t+1}) + R_{t+1} - V(S_t))^2} \\
     &= \sum_{t=0}^{\infty} \EE{\mathbbm{1}(S_t = s)\EE{(V(S_{t+1}) + R_{t+1} - V(S_t))^2 | S_t}}\\
     &= \sum_{t=0}^{\infty} \EE{\mathbbm{1}(S_t = s)\OSV{S_t} }\\
     &= \EE{N(s)}\OSV{s}
 \end{align*}
where (a) also follows from Lemma \ref{lem:uncorrelated}.
        
    \end{itemize}

    From the Central Limit Theorem, we obtain:
    \begin{equation}
        \label{eq:rescaled_CLT}
        \sqrt{n} \cdot \dfrac{1}{n}\sum_{i = 1}^{n} \left(\sum_{t=0}^{\infty} \mathbbm{1}(S_t^{(i)} = s)(V(S_{t+1}^{(i)}) + R_{t+1}^{(i)} - V(S_t^{(i)}))\right)_{s \in \mathcal{S}} \Rightarrow \mathcal{N}\left(0 , {\rm Diag}\left((\EE{N(s)}\OSV{s})_{s \in \mathcal{S}} \right)\right)
    \end{equation}
    
\end{itemize}

Combining (\ref{eq:scale_SLN}) and (\ref{eq:rescaled_CLT}) gives the result. 

\end{proof}

Finally, we just need to take the dot product of $\sqrt{n}\cdot \left(\hat{\mathcal{T}}V(s) - V(s)\right)_{s \in \mathcal{S}}$ with the weighted occupancy measure $\eta_\pi$ to get the result:
\begin{equation*}
    	\sqrt{n}\sum_{s \in \mathcal{S}} \pi(s) \sum_{t=0}^{\infty} P^t(\hat{\mathcal{T}}V-V)(s) \Rightarrow \mathcal{N}\left( 0, \sum_{s \in \mathcal{S}} \frac{\eta_\pi(s)^2}{\EE{N(s)}}\OSV{s}\right)
\end{equation*}

\end{proof}

\subsection{Proof of Theorem \ref{thm:state_coeff}}
The proof follows directly from  Proposition \ref{lem:MC_variance} and Proposition \ref{lem:TD_variance}.

From, Proposition \ref{lem:MC_variance}, we have 
$$\lim_{n \rightarrow \infty} \sqrt{n}\cdot\EE{\left(\hat{V}_{\rm MC}(s) - V(s)\right)^2} = \dfrac{1}{\pr{s \in \tau}}\sum_{s' \in \mathcal{S}}\EE{N(s') | S_0 = s} \OSV{s'}.$$
Similarly, from Proposition \ref{lem:TD_variance}, we have $$\lim_{n \rightarrow \infty} \sqrt{n} \cdot \EE{\left(\hat{V}_{\rm TD}(s) - V(s)\right)^2} = \sum_{s' \in \mathcal{S}} \dfrac{\EE{N(s') | S_0 = s}^2 \OSV{s'}}{\EE{N(s')}}.$$

Taking the ratio of these two limits, we obtain:

\begin{align*}
\lim_{n\to\infty} \dfrac{\EE{\left(\hat{V}_{\rm TD}(s) - V(s)\right)^2}}{\EE{\left(\hat{V}_{\rm MC}(s) - V(s)\right)^2}} &= \dfrac{\sum_{s' \in \mathcal{S}} \EE{N(s') \mid S_0 = s} \OSV{s'}}{\sum_{s' \in \mathcal{S}}\EE{N(s') \mid S_0 = s} \OSV{s'}}\cdot \dfrac{\EE{N(s') \mid S_0 = s} \pr{s \in \tau} }{\EE{N(s')}}\\
&= \EE[s' \sim \mu(s)]{\dfrac{\EE{N(s') \mid S_0 = s} \pr{s \in \tau} }{\EE{N(s')} }}
\end{align*}
where $\mu(s)$ is defined in Definition \ref{def:traj_pooling_coeff}. 

Finally, 
\begin{equation*}
   \EE{N(s') \mid S_0 = s} \pr{s \in \tau} = \EE{N(s \to s')}.
\end{equation*}


\subsection{Proof of Theorem \ref{thm:horizon_truncation}}

 We define $\pi$ to be such that $\pi(s) = 1$, $\pi(s') = -1$ and $\pi(\tilde{s}) = 0$ for all $\tilde{s}$. That way $J(\pi) = V(s) - V(s') = \mathbb{A}(s,s')$. This allows to use Lemma \ref{lem:TD_variance}:

\begin{equation*}
	\lim_{n \rightarrow \infty} n \cdot \EE{\left(\hat{J}_{\rm TD}(\pi) - J(\pi) \right)^2 } = \sum_{\tilde{s} \in \mathcal{S}} \dfrac{\left(\EE{N(\tilde{s}) \mid S_0 = s} - \EE{N(\tilde{s}) \mid S_0 = s'}\right)^2\OSV{\tilde{s}}}{\EE{N(\tilde{s})}}
\end{equation*}

We decompose the sum into three terms that we bound separately:
\begin{equation}
    \begin{split}
	\lim_{n \rightarrow \infty} n \cdot \EE{\left(\hat{J}_{\rm TD}(\pi) - J(\pi) \right)^2 }  & \leq \left(\max_{\tilde{s} \in \mathcal{S}}\left| \dfrac{\EE{N(\tilde{s}) \mid S_0 = s} - \EE{N(\tilde{s}) \mid S_0 = s'}}{\EE{N(\tilde{s})}} \right|\right) \\ & \quad  \times\left(\max_{\tilde{s} \in \mathcal{S}}\OSV{\tilde{s}} \right) \\ 
    &\quad   \times \sum_{\tilde{s} \in \mathcal{S}} \left|\EE{N(\tilde{s}) \mid S_0 = s} - \EE{N(\tilde{s}) \mid S_0 = s'}\right|.
    \label{eq:bounding_horizon_trunc}
     \end{split}
\end{equation}

We start by proving that if trajectories where $S_0$ is sampled from the initial distribution visit $s$ frequently often, then the occupancy measure $\EE{N(\tilde{s})}$ and $\EE{N(\tilde{s}) \mid S_0 = s}$ cannot differ too much (and symmetrically for $s'$). 

If $\EE{N(\tilde{s}) \mid S_0 = s} \geq \EE{N(\tilde{s}) \mid S_0 = s'}$ then

\begin{align*}
    \left| \dfrac{\EE{N(\tilde{s}) \mid S_0 = s} - \EE{N(\tilde{s}) \mid S_0 = s'}}{\EE{N(\tilde{s})}} \right| \leq  \dfrac{\EE{N(\tilde{s}) \mid S_0 = s}}{\pr{s \in \tau} \EE{N(\tilde{s}) \mid S_0 = s}} = \dfrac{1}{\pr{s \in \tau}}.
\end{align*}

When $\EE{N(\tilde{s}) \mid S_0 = s} < \EE{N(\tilde{s}) \mid S_0 = s'}$, a symmetric argument ensures that 
\begin{align*}
    \left| \dfrac{\EE{N(\tilde{s}) \mid S_0 = s} - \EE{N(\tilde{s}) \mid S_0 = s'}}{\EE{N(\tilde{s})}} \right| \leq \dfrac{1}{\pr{s' \in \tau}}.
\end{align*}

Combining these two cases gives: 
\begin{align*}
    \max_{\tilde{s} \in \mathcal{S}}\left| \dfrac{\EE{N(\tilde{s}) \mid S_0 = s} - \EE{N(\tilde{s}) \mid S_0 = s'}}{\EE{N(\tilde{s})}} \right| \leq \dfrac{1}{\min \left(\pr{s \in \tau}, \pr{s' \in \tau} \right)}.
\end{align*}

Plugging this result in Equation \ref{eq:bounding_horizon_trunc}:

\begin{equation}
    \label{eq:bounding_horizon_trunc_simplified}
	\lim_{n \rightarrow \infty} n \cdot \EE{\left(\hat{J}_{\rm TD}(\pi) - J(\pi) \right)^2 } \leq \dfrac{\max_{\tilde{s} \in \mathcal{S}}\OSV{\tilde{s}}}{\min \left(\pr{s \in \tau}, \pr{s' \in \tau} \right)} \cdot \sum_{\tilde{s} \in \mathcal{S}} \left|\EE{N(\tilde{s}) \mid S_0 = s} - \EE{N(\tilde{s}) \mid S_0 = s'}\right|.
\end{equation}

Finally, we show that the last sum scales as the crossing time. To give intuition on why this holds, note that we have the following identity:

\begin{equation*}
    \sum_{\Tilde{s} \in \mathcal{S}} \EE{N(\tilde{s}) \mid S_0 = s} = \EE{T | S_0 = s}.
\end{equation*}

$\EE{N(\tilde{s}) \mid S_0 = s}$ is the expected number of visits to $\tilde{s}$ whens starting at $s$: summing over $\Tilde{s}$ is the expected total number of states visited when starting at $s$ (not counting the terminating state $\emptyset$) which is exactly the horizon. In the case of the comparison of trajectories, $\left|\EE{N(\tilde{s}) \mid S_0 = s} - \EE{N(\tilde{s}) \mid S_0 = s'}\right|$ is how many additional times $\Tilde{s}$ has been visited by one of the trajectories compared to the other, in expectation. Summing over $\Tilde{s}$ gives the expected total number of states that have been visited by only one of the trajectories, which is at most twice the number of states visited before both trajectories reach a common state.

\begin{lemma}
    \label{lem:diff_occ_measure}
	\begin{equation*}
		 \sum_{\tilde{s} \in \mathcal{S}}   |\EE{N(\tilde{s}) \mid S_0 = s} - \EE{N(\tilde{s}) \mid S_0 = s'} | \leq 2H(s,s')
	\end{equation*}
\end{lemma}

\begin{proof}

Let $\tau, \tau' = (S_0, R_1, \dots, S_{T-1}, R_{T}, \emptyset), (S_0', R_1', \dots, S_{T-1}', R_{T'}', \emptyset) \sim \psi$ where $\psi$ is a joint distribution such that $\tau$ and $\tau'$ are marginally distributed like trajectories generated with $S_0 = s$ and $S'_0 = s'$, respectively. We define the crossing time for $\psi$ is defined as the first time a state has been visited by both trajectories: $$H_{\psi}(s,s') = \inf \{t | \{S_0, \dots, S_t\} \cap \{\tilde{S}_0, \dots, \tilde{S}_t\} \neq \emptyset \}.$$ Since we always consider the crossing time of $s$ and $s'$, we omit the $s$ and $s'$ dependency and simply write $H_{\psi}$ for convenience in the rest of the proof. By definition, either $S_{H_\psi}' \in \{S_1, \dots, S_{H_\psi}\}$ or $S_{H_\psi} \in \{S_1', \dots, S_{H_\psi}'\}$. Without loss of generality, we assume $S_{H_\psi}' \in \{s_1, \dots, s_{H_\psi}\}$ holds. Let $N_\psi = \inf \{t | S_t = S_{H_\psi}'\}$, that is $S_{N_\psi} = S_{H_\psi}'$.

We now construct a new trajectory that follows $\tau'$ until the crossing state $S_{N_\psi} = S_{H_\psi}'$ is reached and then follows the trajectory $\tau$: $\hat{\tau} = (S_0', R_1', \dots, S_{H_\psi}' = S_{N_\psi}, R_{N+1}, S_{N+1}, \dots, S_{T-1}, R_{T}, \emptyset)$. By Markov property, $\hat{\tau}$ is identically distributed as $\tau'$.

We are interested in bounding the difference in occupancy measure:
\begin{align*}
    \EE{N(\tilde{s}) \mid S_0 = s} - \EE{N(\tilde{s}) \mid S_0 = s'} &= \EE{\sum_{t=0}^\infty  \mathbbm{1}(S_t = \tilde{s})} -\EE{\sum_{t=0}^\infty  \mathbbm{1}(S_t' = \tilde{s})}.\\
\end{align*}

Using that $\hat{\tau}$ and $\tau'$ are identically distributed, this expression can be rewritten as 
\begin{align*}
     \EE{N(\tilde{s}) \mid S_0 = s} - \EE{N(\tilde{s}) \mid S_0 = s'} &= \EE{\sum_{t=0}^\infty  \mathbbm{1}(S_t = \tilde{s})} -\EE{\sum_{t=0}^\infty  \mathbbm{1}(\hat{S}_t = \tilde{s})}\\
    &= \EE{\sum_{t=0}^{N_\psi-1}  \mathbbm{1}(S_t = \tilde{s})} -\EE{\sum_{t=0}^{H_\psi-1}  \mathbbm{1}(S_t' = \tilde{s})}\\
\end{align*}

Taking absolute values and summing over $\tilde{s}$ gives
\begin{align*}
    \sum_{\tilde{s} \in \mathcal{S}} \left|\EE{N(\tilde{s}) \mid S_0 = s} - \EE{N(\tilde{s}) \mid S_0 = s'}  \right| & = \sum_{\tilde{s} \in \mathcal{S}}\EE{  \left|\sum_{t=0}^{N_\psi-1}  \mathbbm{1}(S_t = \tilde{s}) - \sum_{t=0}^{H_\psi-1}  \mathbbm{1}(S_t' = \tilde{s})\right|} \\
    & \leq \sum_{s' \in \mathcal{S}}\EE{\sum_{t=0}^{N_\psi-1}  \mathbbm{1}(S_t = \tilde{s})} + \sum_{\tilde{s} \in \mathcal{S}}\EE{\sum_{t=0}^{H_\psi-1}  \mathbbm{1}(S_t = s')} \\
    &= \EE{N_\psi} + \EE{H_\psi} \\
    &\leq 2\EE{H_\psi}
\end{align*}
Taking the infimum over all $\psi \in \Psi(s,s')$ that conserves marginal distribution gives the result. 
\end{proof}	

Finally, plugging Lemma \ref{lem:diff_occ_measure} in Equation \ref{eq:bounding_horizon_trunc_simplified} leads to the result:

\begin{equation*}
	\lim_{n \rightarrow \infty} n \cdot \EE{\left(\hat{J}_{\rm TD}(\pi) - J(\pi) \right)^2 } \leq 2\dfrac{\max_{\tilde{s} \in \mathcal{S}}\OSV{\tilde{s}}}{\min \left(\pr{s \in \tau}, \pr{s' \in \tau} \right)} \cdot H(s,s')
\end{equation*}
\subsubsection{Proof of Proposition \ref{prop:MC_lower_bound}}
We now prove that the MC estimate of the advantage can scale as the full horizon, no matter how small the crossing time is.  
We focus on the advantage of $s$ over $s'$ when $\pr{s \in \tau \wedge s' \in \tau} = 0$, that is no trajectory can visit both $s$ and $s'$. No trajectories visiting both $s$ and $s'$ means that the MC value estimates at $s$ and $s'$ are independent since they rely on disjoints sets of trajectories. In terms of variance, this implies:

\begin{equation*}
    \Var{\hat{\mathbb{A}}_{\rm MC}(s,s')} = \Var{\hat{V}_{\rm MC}(s)- \hat{V}_{\rm MC}(s')} =  \Var{\hat{V}_{\rm MC}(s)} +  \Var{\hat{V}_{\rm MC}(s')}.
\end{equation*}

It now suffices to prove that for any $s$,

\begin{equation*}
    \lim_{n \to \infty} \Var{\hat{V}_{\rm MC}(s)} \geq \dfrac{\sigma_{\min}^2}{\pr{s \in \tau}} \EE{T | S_0 = s}.
\end{equation*}

From Proposition \ref{lem:MC_variance}, we know the variance of the MC estimate is

\begin{equation*}
        \lim_{n \to \infty} n\cdot\Var{\hat{V}_{\rm MC}(s)} = \dfrac{1}{\pr{s\in \tau}}\sum_{\tilde{s} \in \mathcal{S}}\EE{N(\tilde{s}) \mid S_0 = s} \OSV{\tilde{s}}.
\end{equation*}

Using that $\OSV{\tilde{s}} \geq \sigma_{\min}^2$ for all $\tilde{s} \in \mathcal{S}$, this expression simplifies to:

\begin{equation*}
        \lim_{n \to \infty} n\cdot\Var{\hat{V}_{\rm MC}(s)} \geq \dfrac{\sigma_{\min}^2}{\pr{s\in \tau}}\sum_{ \in \mathcal{S}}\EE{N(\tilde{s}) \mid S_0 = s}.
\end{equation*}

Finally, 
\begin{align*}
    \sum_{s' \in \mathcal{S}}\EE{N(s') \mid S_0 = s} &= \sum_{\tilde{s} \in \mathcal{S}} \EE{\sum_{t=0}^{T} \mathbbm{1}(S_t = \tilde{s}) | S_0 = s}\\
     &= \EE{\sum_{t=0}^{T} \sum_{\tilde{s} \in \mathcal{S}} \mathbbm{1}(S_t = \tilde{s}) | S_0 = s}\\
    &= \EE{T | S_0 = s}
\end{align*}

\subsection{TD always improves over MC}
\label{subsec:TD_improves}
     Theorem \ref{thm:horizon_truncation} and Proposition \ref{prop:MC_lower_bound} show that the error of the TD estimate of the advantage scales with the crossing time while the error of the MC estimate of the advantage scales with the full horizon, hinting that TD can provide significantly more precise advantage estimates. However, in some cases, as described in Figure \ref{fig:bad_instance}, the crossing time is as large as the full horizon. One can then wonder if there are any setting in which MC can outperform TD for advantage estimation. Theorem \ref{thm:state_coeff} shows that this is never the case for value-to-go estimation. For completeness, we show that, under a technical condition, this is also never the case for advantage estimation.

    \begin{proposition}
        For $s,s' \in \mathcal{S}$ such that $\pr{s \in \tau \wedge s' \in \tau} = 0$,
        \begin{equation*}
            \lim_{n \to \infty} \frac{\EE{\left(\hat{\mathbb{A}}_{\rm TD}(s,s') - \mathbb{A}(s,s')\right)^2}}{\EE{\left(\hat{\mathbb{A}}_{\rm MC}(s,s') - \mathbb{A}(s,s')\right)^2} } \leq 1
        \end{equation*}
    \end{proposition}
    \begin{proof}
        From Lemma \ref{lem:TD_variance} and the proof of Proposition \ref{prop:MC_lower_bound}, we have:
        \begin{align*}
            \lim_{n \rightarrow \infty} n \cdot \EE{\left(\hat{J}_{\rm TD}(\pi) - J(\pi) \right)^2 } &= \sum_{\tilde{s} \in \mathcal{S}} \dfrac{\left(\EE{N(\tilde{s}) \mid S_0 = s} - \EE{N(\tilde{s}) \mid S_0 = s'}\right)^2}{\EE{N(\tilde{s})}}\OSV{\tilde{s}}\\
            \lim_{n \rightarrow \infty} n \cdot \EE{\left(\hat{J}_{\rm MC}(\pi) - J(\pi) \right)^2 } &=\sum_{\tilde{s} \in \mathcal{S}}\left(\dfrac{\EE{N(\tilde{s}) \mid S_0 = s}}{\pr{s\in \tau}}+\dfrac{\EE{N(\tilde{s}) \mid S_0 = s'}}{\pr{s'\in \tau}}\right) \OSV{\tilde{s}}
        \end{align*}
        Since $\pr{s \in \tau \wedge s' \in \tau} = 0$, we have:
        \begin{align*}
            \EE{N(\tilde{s})} &\geq \EE{N(\Tilde{s}) | s\in \tau}\pr{s \in \tau} + \EE{N(\Tilde{s}) | s' \in \tau}\pr{s' \in \tau}\\
            &\geq \EE{N(\Tilde{s}) | S_0 = s}\pr{s \in \tau} + \EE{N(\Tilde{s}) | S_0 = s'}\pr{s' \in \tau}
        \end{align*}
        Plugging this into the limiting error of the TD estimates, we have:
        \begin{align*}
            \lim_{n \rightarrow \infty} n \cdot \EE{\left(\hat{J}_{\rm TD}(\pi) - J(\pi) \right)^2 } &\leq \sum_{\tilde{s} \in \mathcal{S}} \dfrac{\left(\EE{N(\tilde{s}) \mid S_0 = s} - \EE{N(\tilde{s}) \mid S_0 = s'}\right)^2}{\EE{N(\Tilde{s}) | S_0 = s}\pr{s \in \tau} + \EE{N(\Tilde{s}) | S_0 = s'}\pr{s' \in \tau}}\\
            &\leq \sum_{\tilde{s} \in \mathcal{S}} \max \left(\dfrac{\EE{N(\tilde{s}) \mid S_0 = s}}{\pr{s \in \tau}} + \dfrac{\EE{N(\tilde{s}) \mid S_0 = s'}}{\pr{s' \in \tau}}\right)\\
            &\leq \lim_{n \rightarrow \infty} n \cdot \EE{\left(\hat{J}_{\rm MC}(\pi) - J(\pi) \right)^2 }
        \end{align*}
    \end{proof}

\subsection{On tightness of bounds of advantage estimation}
\label{subsec:tight_bounds}
We start by showing that the lower bound on the MSE of the MC estimate of the advantage is tight.
\begin{proposition}
    \label{prop:MC_lower_bound_tight}
    Let $\mathcal{M}$ be a MRP such that $\EE{R(s,s')} = 0$, $\Var{R(s,s')} = V$ for all $s,s'$. For all $s,s' \in \mathcal{S}$ such that $\pr{s \in \tau \wedge s' \in \tau} = 0$, 
    \begin{align*}
        \lim_{n \to \infty} n \cdot \EE{\left(\hat{\mathbb{A}}_{\rm MC}(s,s') - \mathbb{A}(s,s')\right)^2} 
       = \sigma_{\rm \min}^2\left(\frac{\EE{T | S_0 = s}}{\pr{s\in \tau}} + \frac{\EE{T | S_0 = s'}}{\pr{s' \in \tau}} \right).
    \end{align*}
\end{proposition}
\begin{proof}
    The only inequality in the proof of Proposition \ref{prop:MC_lower_bound} is 
    \begin{equation*}
        \OSV{\tilde{s}} \geq \sigma_{\min}^2
    \end{equation*}
    We show that this inequality is an equality when $\EE{R(s,s')} = 0$ and $\Var{R(s,s')} = V$ for all $s,s'$. 
    In particular, that $V(s) = 0$ for all $s$. 
    \begin{align*}
        \OSV{\tilde{s}} &= \Var{R_0 + V(S_1) | S_0 = s}\\
        &= \EE{\Var{R_0 + V(S_1) | S_0 = s, S_1} | S_0 = s} + \Var{\EE{R_0 + V(S_1)|S_0 = s} | S_0 = s}\\
        &=  \EE{\Var{R(S_0,S_1)} | S_0 = s}\\
        &= V
    \end{align*}
    Hence, $\OSV{\tilde{s}}$ is constant equal to $V$. We therefore have the equality 
    \begin{equation*}
        \OSV{\tilde{s}} = \sigma_{\min}^2.
    \end{equation*}
\end{proof}

In the case of the advantage estimation using TD learning, we can prove a lower bound on the MSE. The gap between the lower bound and the upper bound is a factor 2.
\begin{proposition}

    There exists a MRP $\mathcal{M} = \{\mathcal{S}\cup \{\emptyset\}, P, R ,d \}$ and $s,s' \in \mathcal{S}$ such that $\pr{s \in \tau \wedge s' \in \tau} = 0$ and
    \begin{align*}
    \lim_{n\rightarrow\infty} n\cdot \EE{\left(\hat{\mathbb{A}}_{\rm TD}(s,s') - \mathbb{A}(s,s')\right)^2}  \geq  \sigma^2_{\rm max}\left(\dfrac{1}{\min\left(\pr{ s\in \tau},\pr{s' \in \tau}\right)}\right) \cdot H(s,s')
\end{align*}
\end{proposition}
\begin{proof}

We define formally a set of MRP on which the bound is achieved. The structure is pictured in Figure \ref{fig:meeting_horizon}. Let $\mathcal{M}_{W,T,H}$ be the MRP with 
\begin{itemize}
    \item States $\mathcal{S}_{W,T,H} = \{s_{(w,t)} | w \in \{1,\dots, W\}, t\in \{1,\dots, H-1\}\} \} \cap \{s_t | t \in \{H, \dots, T-1\}\}$.
    \item Transitions $P(s_{(w,t+1)} | s_{(w,t)}) = 1$ for all $1\leq w \leq W, 1\leq t \leq H-2$, $P(s_{H} | s_{(w,H-1)}) = 1$ for all $1 \leq 1 \leq W$,  $P(s_{t+1} | s_{t}) = 1$ for all $H \leq t \leq T-2$ and  $P(s_{T-1} | \emptyset) = 1$
    \item $\Var{R(s,s')} = V$ and $\EE{R(s,s')} = 0$ for all $s,s'$ 
\end{itemize}

First, we note that the TD estimate of the advantage between $s_{(i,1)}$ and $s_{(j,1)}$ in $\mathcal{M}_{W,T,H}$ has the same distribution as the MC estimate of the advantage in $\mathcal{M}_{W,H,H}$. However, from Proposition \ref{prop:MC_lower_bound_tight}, we know that 

\begin{align*}
        \lim_{n \to \infty} n \cdot \EE{\left(\hat{\mathbb{A}}_{\rm MC, \mathcal{M}_{W,H,H}}(s_{(i,1)},s_{(j,1)}) - \mathbb{A}_{\mathcal{M}_{W,H,H}}(s_{(i,1)},s_{(j,1)})\right)^2} = HV\left(\frac{1}{\pr{s_{(i,1)} \in \tau}} + \frac{1}{\pr{s_{(j,1)} \in \tau}} \right).
\end{align*}

Noting that $\mathbb{A}_{\mathcal{M}_{W,H,H}}(s_{(i,1)},s_{(j,1)}) = \mathbb{A}_{\mathcal{M}_{W,T,H}}(s_{(i,1)},s_{(j,1)})$ (since all value-to-go are zero), we obtain:
\begin{align*}
        \lim_{n \to \infty} n \cdot \EE{\left(\hat{\mathbb{A}}_{\rm TD, \mathcal{M}_{W,T,H}}(s,s') - \mathbb{A}(s,s')\right)^2} &= HV\left(\frac{1}{\pr{s_{(i,1)} \in \tau}} + \frac{1}{\pr{s_{(j,1)} \in \tau}} \right)\\
        &\geq HV\dfrac{1}{\min\left(\pr{s_{(i,1)} \in \tau},\pr{s_{(j,1)} \in \tau}\right)}\\
        &\geq  \sigma^2_{\rm max}\left(\dfrac{1}{\min\left(\pr{ s_{(i,1)}\in \tau},\pr{s_{(j,1)} \in \tau}\right)}\right) \cdot H(s_{(i,1)},s_{(j,1)})
\end{align*}
\end{proof}

As mentioned in Section \ref{sec:horizon_truncation}, the upper bound on the MSE of TD estimates does not take into account any pooling that can happen before the crossing time. Hence, it is generally loose. The tightness result is achieved on MRPs with no pooling before the crossing time. On the other hand, the tightness of the lower bound on MC estimates only requires to control the distribution of the rewards and the transition structure of the MRP is not important.

\end{appendix}

\end{document}